\documentclass{article}

\usepackage{microtype}
\usepackage{graphicx}
\usepackage{subfigure}
\usepackage{booktabs} 
\usepackage{enumitem}
\usepackage{pifont}
\usepackage{multirow}
\usepackage{makecell}

\usepackage{hyperref}

\usepackage[preprint]{icml2023}

\usepackage{amsmath}
\usepackage{amssymb}
\usepackage{mathtools}
\usepackage{amsthm}
\allowdisplaybreaks
\usepackage[normalem]{ulem}

\usepackage[capitalize,noabbrev]{cleveref}

\theoremstyle{plain}
\newtheorem{theorem}{Theorem}[section]
\newtheorem{proposition}[theorem]{Proposition}
\newtheorem{lemma}[theorem]{Lemma}
\newtheorem{corollary}[theorem]{Corollary}
\theoremstyle{definition}
\newtheorem{definition}[theorem]{Definition}

\theoremstyle{remark}

\usepackage[textsize=tiny]{todonotes}

\icmltitlerunning{Spectral GNN via Two-dimensional (2-D) Graph Convolution}

\begin{document}

\twocolumn[
\icmltitle{Spectral GNN via Two-dimensional (2-D) Graph Convolution}




\begin{icmlauthorlist}
\icmlauthor{Guoming Li}{idp,mbz}
\icmlauthor{Jian Yang}{casia}
\icmlauthor{Shangsong Liang}{mbz}
\icmlauthor{Dongsheng Luo}{fiu}
\end{icmlauthorlist}

\icmlaffiliation{idp}{Independent Researcher}
\icmlaffiliation{mbz}{MBZUAI}
\icmlaffiliation{fiu}{Florida International University}
\icmlaffiliation{casia}{Institute of Automation, Chinese Academy of Sciences}

\icmlcorrespondingauthor{Guoming Li}{paskardli@outlook.com}

\icmlkeywords{Machine Learning, ICML}

\vskip 0.3in
]



\printAffiliationsAndNotice 

\begin{abstract}
Spectral Graph Neural Networks (GNNs) have achieved tremendous success in graph learning. 
As an essential part of spectral GNNs, spectral graph convolution extracts crucial frequency information in graph data, leading to superior performance of spectral GNNs in downstream tasks. 
However, in this paper, we show that existing spectral GNNs remain critical drawbacks in performing the spectral graph convolution. 
Specifically, considering the spectral graph convolution as a construction operation towards target output, we prove that existing popular convolution paradigms cannot construct the target output with mild conditions on input graph signals, causing spectral GNNs to fall into suboptimal solutions. 
To address the issues, we rethink the spectral graph convolution from a more general two-dimensional (2-D) signal convolution perspective and propose a new convolution paradigm, named 2-D graph convolution. 
We prove that 2-D graph convolution unifies existing graph convolution paradigms, and is capable to construct arbitrary target output.
Based on the proposed 2-D graph convolution, we further propose ChebNet2D, an efficient and effective GNN implementation of 2-D graph convolution through applying Chebyshev interpolation. 
Extensive experiments on benchmark datasets demonstrate both effectiveness and efficiency of the ChebNet2D. 
\end{abstract}


\section{Introduction}
\label{section-introduction}

Graph Neural Networks (GNNs)~\cite{comprehensivegnn} have emerged as a powerful tool in machine learning, adept at handling structured data represented as graphs. 
Among these, spectral GNNs, which process graph data as signals on graph (namely graph signals) in the frequency domain, have shown remarkable performance in various graph learning tasks~\cite{spectralGNN-1,spectralGNN-2,spectralGNN-3,spectralGNN-4,spectralGNN-5,spectralGNN-6}. 
These methods leverage the principles of Graph Signal Processing (GSP)~\cite{GraphSignalProcessingOverviewChallengesandApplications} to perform the key spectral graph convolution~\cite{graphconv_1} to handle the graph signals, significantly enhancing learning outcomes. 
However, a notable gap exists in this spectral graph convolution, the core component of spectral GNNs. 
In the conventional GSP domain, the spectral graph convolution is designed on vector-type graph signals~\cite{graphconv_2,GraphSignalProcessingforMachineLearningAReviewandNewPerspectives}. 
In contrast, practical graph learning tasks often deal with matrix-type signals, such as node feature matrices~\cite{representationlearningongraph,comprehensivegnn,deepgraphlearningsurvey,surveygraphstructurelearning}. 
This discrepancy leads to varying convolution paradigms in existing spectral GNNs that have been implemented (discussed in Section~\ref{section-notations-and-preliminaries-spectral-graph-neural-networks}), which further raises a critical question: 
\par \textit{Are those existing spectral graph convolution paradigms appropriate for practical matrix-type graph signals?} 
\par Motivated by this question, in this paper, we revisit the popular convolution paradigms of existing spectral GNNs and show the critical issues of those paradigms. 
Specifically, considering the spectral graph convolution as construction operation to the target output, we prove that existing convolution paradigms cannot construct the target output when mild conditions on input graph signals hold, which further causes spectral GNNs to fall into suboptimal solutions. 
\par To address the critical issues of existing convolution paradigms, we rethink the spectral graph convolution from a more general two-dimensional (2-D) signal convolution perspective, and propose a new convolution paradigm, named 2-D graph convolution. 
Theoretically, we prove that the proposed 2-D graph convolution unifies the popular spectral graph convolution paradigms, and is always capable of constructing arbitrary target output. 
Moreover, we prove that the parameter number of 2-D graph convolution is irreducible for constructing arbitrary target output, which examines the parameter efficiency of our 2-D graph convolution. 
\par Further, we propose ChebNet2D, a novel spectral GNN applying the 2-D graph convolution. 
ChebNet2D implements the convolution operators in 2-D graph convolution with Chebyshev interpolation, leading to a noticeably streamlined model architecture with marginal complexity increase against other spectral GNNs. 
By conducting node classification experiments on $18$ datasets, we show that ChebNet2D outperforms SOTA methods, verifying the superiority of 2-D graph convolution against other convolution paradigms. 


\section{Notations and Preliminaries}
\label{section-notations-and-preliminaries}
In this paper, we use the boldface lowercase letters for \textit{vectors}, e.g., $\boldsymbol{f}$, where $\boldsymbol{f}_{i}$ denotes the $i$-th element of $\boldsymbol{f}$. 
\textit{Matrices} are denoted with boldface uppercase letters, e.g., $\boldsymbol{F}$, where $\boldsymbol{F}_{i:}$ and $\boldsymbol{F}_{:j}$ denote the $i$-th row and $j$-column of $\boldsymbol{F}$ respectively, and $\boldsymbol{F}_{ij}$ denotes the element of $\boldsymbol{F}$ on $i$-th row, $j$-column. 
Let $\mathcal{G}=(\mathcal{V}, \mathcal{E}, \boldsymbol{X})$ denote an undirected graph $\mathcal{G}$ with a finite node set $\mathcal{V}=\left\{v_{1},v_{2},...,v_{N} \right\}$, an edge set $\mathcal{E}\subseteq\mathcal{V}\times\mathcal{V}$, and a node feature matrix $\boldsymbol{X}\in\mathbb{R}^{N\times K}$, where $K$ is dimension of node features. 
Let $\boldsymbol{A}\in\left\{0,1\right\}^{N\times N}$ be the unweighted adjacency matrix of $\mathcal{G}$ and $\boldsymbol{D}$ be the diagonal matrix whose diagonal element $\boldsymbol{D}_{ii}$ is the degree of node $v_{i}$. 
Let $\boldsymbol{L}=\boldsymbol{I}-\boldsymbol{D}^{-\frac{1}{2}}\boldsymbol{A}\boldsymbol{D}^{-\frac{1}{2}}$ be the normalized graph Laplacian matrix of $\mathcal{G}$, where $\boldsymbol{I}$ is the identity matrix. 
Let $\boldsymbol{L}=\boldsymbol{U}diag(\boldsymbol{\lambda})\boldsymbol{U}^{T}$ be the eigendecomposition on $\boldsymbol{L}$, where $\boldsymbol{U}\in\mathbb{R}^{N\times N}$, $\boldsymbol{\lambda}\in\mathbb{R}^{N}$ be eigenvectors and eigenvalues of $\boldsymbol{L}$, and $diag(\boldsymbol{\lambda})$ is the diagonal matrix of $\boldsymbol{\lambda}$. 

\subsection{Spectral Graph Convolution in Graph Signal Processing}
\label{section-notations-and-preliminaries-spectral-graph-convolution}
\text{Graph signal processing (GSP)} is developed to analyze and process the challenging \textit{graph signals}, i.e., the signals generated on graphs and accompanied with complicated topologies~\cite{GraphSignalProcessingOverviewChallengesandApplications,GraphSignalProcessingforMachineLearningAReviewandNewPerspectives}. 
As a fundamental technique in GSP, \textit{spectral graph convolution}~\cite{graphconv_1,graphconv_2} processes graph signals by performing convolution operation (also called filtering operation) on them, which is defined based on \textit{graph Fourier transform (GFT)} from spectral graph theory~\cite{spectralgraphtheory}. 
Specifically, given an $N$-dimensional graph signal vector $\boldsymbol{f}\in\mathbb{R}^{N}$ on graph $\mathcal{G}$, the GFT of $\boldsymbol{f}$ and its inverse transform are formulated as below:
\begin{align}
\label{equation-GFT-inverseGFT}
&\Hat{\boldsymbol{f}}=\boldsymbol{U}^{T}\boldsymbol{f}\ ,\\
&\boldsymbol{f}=\boldsymbol{U}\Hat{\boldsymbol{f}}\ , 
\end{align}
where $\Hat{\boldsymbol{f}}$ denotes the GFT of $\boldsymbol{f}$, and $\boldsymbol{U}$ is eigenvectors of graph Laplacian. 
The GFT transforms graph signals from the original vertex domain into an alternative \textit{frequency domain}, whilst providing a novel perspective in analyzing graph signals. 
Based on the GFT, the spectral graph convolution can be thereby defined: given graph signal $\boldsymbol{f}\in\mathbb{R}^{N}$ on graph $\mathcal{G}$ and $\boldsymbol{g}\in\mathbb{R}^{N}$ called \textit{spectral graph filter}, the spectral graph convolution of $\boldsymbol{f}$ with filter $\boldsymbol{g}$ is formulated as~\cite{graphconv_1,graphconv_2}
\begin{align}
\label{equation-graphconv-1d}
\boldsymbol{f}\ast_{\mathcal{G}}\boldsymbol{g}=\boldsymbol{U}diag(\boldsymbol{g})\boldsymbol{U}^{T}\boldsymbol{f}\doteq\boldsymbol{\Phi}_{\mathcal{G}}\boldsymbol{f}\ , 
\end{align}
where $\boldsymbol{\Phi}_{\mathcal{G}}=\boldsymbol{U}diag(\boldsymbol{g})\boldsymbol{U}^{T}$ is called \textit{convolution operator} associated with filter $\boldsymbol{g}$ on graph $\mathcal{G}$.

\subsection{Spectral Graph Neural Networks via Spectral Graph Convolution}
\label{section-notations-and-preliminaries-spectral-graph-neural-networks}
As an important branch of GNNs, spectral GNNs process the graph signals through conducting spectral graph convolution operation as Eq.~\ref{equation-graphconv-1d}. 
Notably, the practical graph signals in graph learning tasks are generally in matrix forms (called node feature)~\cite{representationlearningongraph,comprehensivegnn,deepgraphlearningsurvey,surveygraphstructurelearning}, while conventional spectral graph convolution is defined on vector-type signals (see Eq.~\ref{equation-graphconv-1d}). 
This gap leads to diverse spectral graph convolution paradigms proposed in previous spectral GNNs. 
Specifically, let $\boldsymbol{F}\in\mathbb{R}^{N\times C}$, $\boldsymbol{Z}\in\mathbb{R}^{N\times C}$ be input graph signal matrix and convolution output, there are three popular convolution paradigms in previous spectral GNNs: 
\begin{itemize}[leftmargin=*,parsep=2pt,itemsep=2pt,topsep=2pt]
\item {\bf Paradigm (I):} 
\begin{equation}
\label{equation-paradigm-single-filter-and-no-feature-interaction}
\boldsymbol{Z}=\boldsymbol{\Phi}_{\mathcal{G}}\boldsymbol{F}\ ,
\end{equation}
where $\boldsymbol{\Phi}_{\mathcal{G}}$ is convolution operator. 
As the most widely-used paradigm of spectral GNNs, Paradigm {\bf (I)} has achieved great success on graph learning tasks~\cite{APPNP,SSGC,GNN-HF-LF,GPRGNN,BernNet-GNN-narrowbandresults-1,ChebNetII,OptBasisGNN}. 
\item {\bf Paradigm (II):} 
\begin{align}
\label{equation-paradigm-single-filter-and-feature-interaction}
\boldsymbol{Z}=\boldsymbol{\Phi}_{\mathcal{G}}\boldsymbol{F}\boldsymbol{R}\ , 
\end{align}
where $\boldsymbol{R}\in\mathbb{R}^{C\times C}$ denotes linear transformation on $\boldsymbol{F}$. 
This linear transformation $\boldsymbol{R}$ models column-wise correlation of the graph signals, leading to better performance of spectral GNNs~\cite{2dGFT-gnn-1,SGC,Adagnn--spectral-2dgnn,DSGC-spectral-2dgnn-2dGFT-gnn-2,specGN}. 
\item {\bf Paradigm (III):} 
\begin{align}
\label{equation-paradigm-individual-filter-and-no-feature-interaction}
\boldsymbol{Z}_{:c}=\boldsymbol{\Phi}_{\mathcal{G}}^{(c)}\boldsymbol{F}_{:c}\ ,\quad c=1,2,...,C\ ,
\end{align}
where $\boldsymbol{\Phi}_{\mathcal{G}}^{(c)}=\boldsymbol{U}diag(\boldsymbol{g}^{(c)})\boldsymbol{U}^{T}$, $c=1,2,...,C$, are individual convolution operators on each column of $\boldsymbol{F}$. 
By applying diverse operators to performing individual filtering on each column of $\boldsymbol{F}$, spectral GNNs as Paradigm {\bf (III)} have achieved promising results~\cite{GWNN,ADC,JacobiConv}. 
\end{itemize}
To match our analysis with majority of research on spectral GNNs, we focus on the mainstream model architecture of spectral GNNs that decouples feature transformation and spectral graph convolution, which means $\boldsymbol{F}$ is obtained by conducting dimension reduction on node feature $\boldsymbol{X}$~\cite{APPNP,SSGC,GNN-HF-LF,GPRGNN,BernNet-GNN-narrowbandresults-1,EvenNet,JacobiConv,ChebNetII,NFGNN,robusempGCN,OptBasisGNN}. 

\section{Two-dimensional (2-D) Graph Convolution}
\label{section-2dgraphconv}
In this section, we first analyze Paradigm {\bf (I)}, {\bf (II)}, {\bf (III)} from the perspective of target construction and prove that target output can \textbf{never} be constructed by those paradigms in certain cases. 
After that, we rethink the target construction problem from the perspective of 2-D signal convolution and propose our 2-D graph convolution. 
With rigorous theoretical analysis, we prove that 2-D graph convolution is a generalized paradigm of Paradigm {\bf (I)}, {\bf (II)} and {\bf (III)}, and can construct arbitrary output with $0$ error. 

\subsection{Analyzing Convolution Paradigms in the Lens of Target Output Construction}
\label{section-2dgraphconv-motivation-analyzing-existing-paradigm}
By decomposing feature transformation and convolution operation, we restate the spectral graph convolution as a construction task to the target output with graph convolution operation. 
Formally, let $\boldsymbol{F}$ be graph signal obtained with feature transformation on $\boldsymbol{X}$ and $\mathcal{T}_{\boldsymbol{\Psi}}$ with parameter $\boldsymbol{\Psi}$ be function that performs graph convolution, the goal of spectral graph convolution is to minimize the error as below: 
\begin{align}
\label{equation-target-output-construction}
||\mathcal{T}_{\boldsymbol{\psi}}(\boldsymbol{F})-\boldsymbol{Z}^{*}||_{F}\ , 
\end{align}
where $\boldsymbol{Z}^{*}$ denotes the target output, and $||\cdot||_{F}$ is Frobenius norm~\cite{linearalgebra}. 
Similar learning goal is introduced in previous works on spectral GNNs~\cite{JacobiConv,OptBasisGNN}, while they focus on achieving minimal rather than $0$ error with strong assumption on $\boldsymbol{Z}^{*}$, i.e., $\boldsymbol{Z}^{*}$ can be constructed with existing convolution paradigms. 
In contrast, our problem statement aims to achieve a more challenging setting with $0$ error to arbitrary $\boldsymbol{Z}^{*}$. 
\par Based on this problem statement, we analyze the capability of Paradigm {\bf (I)}, {\bf (II)} and {\bf (III)} from the perspective of target output construction, and obtain the following theorem: 
\begin{theorem}
\label{theorem-existing-paradigms-fails-to-construct-optimal-output}
Let $\mathcal{T}_{\boldsymbol{\Psi}}$ be the parameterized function that performs graph convolution as Paradigm {\bf (I)} or {\bf (II)} or {\bf (III)} on graph signal matrix $\boldsymbol{F}\in\mathbb{R}^{N\times C}$, where $\boldsymbol{\Psi}$ denotes the parameters. 
For the $\boldsymbol{F}$ which is nontrivial on frequency, i.e., no zero row exists in $\boldsymbol{U}^{T}\boldsymbol{F}$, the convolution output $\mathcal{T}_{\boldsymbol{\Psi}}(\boldsymbol{F})$ fails to construct the target output $\boldsymbol{Z}^{*}$ with certain conditions on $\boldsymbol{F}$, $\boldsymbol{Z}^{*}$. 
That is, in certain cases, the construction error $||\mathcal{T}_{\boldsymbol{\Psi}}(\boldsymbol{F})-\boldsymbol{Z}^{*}||_{F}$ cannot achieve $0$ for all parameter values of $\boldsymbol{\Psi}$. 
\end{theorem}
We prove this theorem in Appendix~\ref{appendix-proofs-and-derivations-theorem-existing-paradigms-fails-to-construct-optimal-output}. 
Theorem~\ref{theorem-existing-paradigms-fails-to-construct-optimal-output} demonstrates that existing three paradigms contain issues in constructing the target output when certain conditions on $\boldsymbol{F}$ and $\boldsymbol{Z}^{*}$ hold. 
Moreover, note that $\boldsymbol{F}$ is obtained with MLP or linear transformation on $\boldsymbol{X}$, and the target $\boldsymbol{Z}^{*}$ is unseen towards the learning process. 
Thus, both $\boldsymbol{F}$ and $\boldsymbol{Z}^{*}$ contain substantial uncertainties on practical properties, causing unexpected risk of falling into a suboptimal solution to the model with Paradigm {\bf (I)} or {\bf (II)} or {\bf (III)}. 
In particular, we further achieve a corollary as follows: 
\begin{corollary}
\label{corollary-simple-sum-still-failures}
Let $\mathcal{T}^{(1)}$, $\mathcal{T}^{(2)}$ and $\mathcal{T}^{(3)}$ be the parameterized functions performing spectral graph convolution as Paradigm {\bf (I)}, {\bf (II)} and {\bf (III)}, respectively. 
For the spectral graph convolution function constructed as $(\mathcal{T}^{(i)}+\mathcal{T}^{(j)})$, $i,j\in\left\{1,2,3\right\}$, there are certain cases of $\boldsymbol{F}$, $\boldsymbol{Z}^{*}$, making the convolution output $(\mathcal{T}^{(i)}+\mathcal{T}^{(j)})(\boldsymbol{F})$ fails to construct the target output $\boldsymbol{Z}^{*}$ for all parameter values. 
\par Especially, for the convolution function constructed as $(\mathcal{T}^{(1)}+\mathcal{T}^{(2)}+\mathcal{T}^{(3)})$, the cases leading to failure on target output construction still exists. 
\end{corollary}
The proof of this corollary is moved to Appendix~\ref{appendix-proofs-and-derivations-corollary-simple-sum-still-failures}. 
Corollary~\ref{corollary-simple-sum-still-failures} further proves that the drawbacks cannot be tackled through naively increasing complexity with existing convolution paradigms. 
Hence, based on the theoretical results above, developing an advanced paradigm of spectral graph convolution which can avoid potential failures on target output construction is thereby necessary.

\subsection{From 2-D Signal Convolution to 2-D Graph Convolution}
\label{section-2dgraphconv-2d-signalconv-2d-graphconv}
Instead of conventionally considering vector-wise convolution as Eq.~\ref{equation-graphconv-1d}, we view $\boldsymbol{F}$ as a whole two-dimensional (2-D) signal and start with the convolution operation on general 2-D signal. 
Particularly, we consider fully-connected network (FCN), a fundamental framework for processing 2-D signals (e.g., image signals), which can be considered as a general paradigm of convolution on 2-D signals~\cite{deep-learning,FC=Conv-1,FC=Conv-2}. 
Specifically, given input 2-D signal $\boldsymbol{F}$, the output of FCN convolution $\boldsymbol{Z}$ is formulated as follows: 
\begin{equation}
\label{equation-2dconv}
\boldsymbol{Z}_{ij}=\sum_{n=1}^{N}\sum_{c=1}^{C}\boldsymbol{\Omega}_{nc}^{(i,j)}\boldsymbol{F}_{nc}\ .
\end{equation}
Here, $\boldsymbol{\Omega}_{nc}^{(i,j)}\in\mathbb{R}$ is the element of $\boldsymbol{\Omega}^{(i,j)}\in\mathbb{R}^{N\times C}$ denoting the coefficient matrix for generating $\boldsymbol{Z}_{ij}$. 
Numerous works have demonstrated the significant capability and necessity of FCN convolution on 2-D signal processing~\cite{fully-connected-1,fully-connected-2,fully-connected-3,fully-connected-4,fully-connected-5,fully-connected-6}. 
Specially, we derive Eq.~\ref{equation-2dconv} to a more intuitive paradigm as follows: 
\begin{align}
\label{equation-2stepsconv}
\textit{Vec}\left(\boldsymbol{Z}\right)=&
\left[\begin{array}{ccc}
    \boldsymbol{\Phi}^{(1,1)} & \cdots & \boldsymbol{\Phi}^{(C,1)} \\
    \vdots & \ddots & \vdots \\
    \boldsymbol{\Phi}^{(1,C)} & \cdots & \boldsymbol{\Phi}^{(C,C)}
\end{array}\right]
\cdot \textit{Vec}\left(\boldsymbol{F}\right)\ .
\end{align}
Here, $\boldsymbol{\Phi}^{(c,j)}\in\mathbb{R}^{N\times N}$ ($c,j=1,2,...,C$) is coefficient matrix whose elements $\boldsymbol{\Phi}^{(c,j)}_{in}$ are equal to $\boldsymbol{\Omega}_{nc}^{(i,j)}$ in Eq.~\ref{equation-2dconv}; $\textit{Vec}(\cdot)$ denotes the vectorization operator~\cite{linearalgebra}. 
We prove the equivalence between Eq.~\ref{equation-2dconv} and Eq.~\ref{equation-2stepsconv} in Appendix~\ref{appendix-proofs-and-derivations-2dconv-to-2stepsconv}. 
Note that operator $\boldsymbol{\Phi}^{(c,j)}$ in Eq.~\ref{equation-2stepsconv} performs convolution operation to the columns of $\boldsymbol{F}$, which is consistent to the vector-wise convolution as Eq.~\ref{equation-graphconv-1d}. 
Based on this insight, we thereby replace $\boldsymbol{\Phi}^{(c,j)}$ with graph convolution operator $\boldsymbol{\Phi}^{(c,j)}_{\mathcal{G}}=\boldsymbol{U}diag(\boldsymbol{g}^{(c,j)})\boldsymbol{U}^{T}$, thus proposing our 2-D graph convolution defined as follows: 
\begin{definition}
\label{definition-2dgraphconv}
{\bf (2-D Graph Convolution)} The 2-D graph convolution on input graph signal matrix $\boldsymbol{F}\in\mathbb{R}^{N\times C}$ is formulated as follows: 
\begin{align}
\label{equation-2dgraphconv}
\boldsymbol{Z}=&\textit{Vec}^{-1}\left(
\left[\begin{array}{ccc}
    \boldsymbol{\Phi}^{(1,1)}_{\mathcal{G}} & \cdots & \boldsymbol{\Phi}^{(C,1)}_{\mathcal{G}} \\
    \vdots & \ddots & \vdots \\
    \boldsymbol{\Phi}^{(1,C)}_{\mathcal{G}} & \cdots & \boldsymbol{\Phi}^{(C,C)}_{\mathcal{G}}
\end{array}\right]
\cdot \textit{Vec}\left(\boldsymbol{F}\right)\right)\ , 
\end{align}
where $\boldsymbol{\Phi}^{(c,j)}_{\mathcal{G}}=\boldsymbol{U}diag(\boldsymbol{g}^{(c,j)})\boldsymbol{U}^{T}$, $c,j=1,2,...,C$, is graph convolution operator with parameterized graph filter $\boldsymbol{g}^{(c,j)}$, $\boldsymbol{Z}$ denotes convolution output, $\textit{Vec}^{-1}(\cdot)$ denotes the inverse of $\textit{Vec}(\cdot)$ making the equation 
\begin{equation}
\label{equation-vectorization}
(\textit{Vec}\circ \textit{Vec}^{-1})(\boldsymbol{M})=(\textit{Vec}^{-1}\circ \textit{Vec})(\boldsymbol{M})=\boldsymbol{M}\ 
\end{equation}
holds for arbitrary matrix $\boldsymbol{M}$. 
\end{definition}
Similar to Eq.~\ref{equation-2dconv}, the 2-D graph convolution constructs $\boldsymbol{Z}$ through involving all elements of the input $\boldsymbol{F}$, making it a novel spectral graph convolution paradigm on 2-D graph signals. 
In the following section, we will show the proposed 2-D graph convolution is a generalized paradigm of the mentioned paradigms and has significant capability. 

\subsection{Theoretical Investigation of 2-D Graph Convolution}
\label{section-2dgraphconv-theoretical-investigation-of-2dgraphconv}

\subsubsection{2-D graph convolution as generalized convolution operation}
\label{section-2dgraphconv-theoretical-investigation-of-2dgraphconv-regarding-as-generalized}
We refer to the proposed 2-D graph convolution as a generalized graph convolution framework of Paradigm {\bf (I)}, {\bf (II)} and {\bf (III)}, which is summarized in the following proposition: 
\begin{proposition}
\label{proposition-generalized-to-special}
By adjusting the $\boldsymbol{\Phi}^{(c,j)}_{\mathcal{G}}$ with specific constraints, 2-D graph convolution can perform the three graph convolution paradigms, i.e., Paradigm {\bf (I)}, {\bf (II)} and {\bf (III)}. 
\end{proposition}
We prove the Proposition~\ref{proposition-generalized-to-special} holds with derivation in Appendix~\ref{appendix-proofs-and-derivations-proposition-generalized-to-special}. 
Thus, based on the Proposition~\ref{proposition-generalized-to-special}, we show that our 2-D graph convolution unifies the existing popular graph convolution paradigms into one generalized framework, where $\boldsymbol{\Phi}^{(c,j)}_{\mathcal{G}}$ serves as key role in such representation. 

\subsubsection{2-D graph convolution on error-free target output construction}
\label{section-2dgraphconv-theoretical-investigation-of-2dgraphconv-target-output-construction}
With considering the convolution operation as target output construction with input and convolution parameters, we call that the proposed 2-D graph convolution is error-free construction towards arbitrary target output. 
Specifically, our statement is guaranteed by the following theorem: 
\begin{theorem}
\label{theorem-constructing-target-output-embeddings}
Let $\boldsymbol{F}$ and $\boldsymbol{Z}^{*}$ be the same definitions as Theorem~\ref{theorem-existing-paradigms-fails-to-construct-optimal-output}, and $\Bar{\mathcal{T}}_{\Bar{\boldsymbol{\Psi}}}$ denotes the function performing 2-D graph convolution as Eq.~\ref{equation-2dgraphconv}, where $\Bar{\boldsymbol{\Psi}}$ denotes parameter set. 
There always exists at least one parameter value of $\Bar{\boldsymbol{\Psi}}$ (denoted as $\Bar{\boldsymbol{\psi}}$), making $\Bar{\mathcal{T}}_{\Bar{\boldsymbol{\Psi}}}(\boldsymbol{F})|_{\Bar{\boldsymbol{\Psi}}=\Bar{\boldsymbol{\psi}}}=\boldsymbol{Z}^{*}$. 
That is, the 2-D graph convolution is always capable of constructing $\boldsymbol{Z}^{*}$ with $0$ construction error. 
\end{theorem}
The proof is deferred to Appendix~\ref{appendix-proofs-and-derivations-theorem-constructing-target-output-embeddings}. 
Theorem~\ref{theorem-constructing-target-output-embeddings} shows that the representation range of 2-D graph convolution always covers the target output, demonstrating the significant capability of 2-D graph convolution on (arbitrary) target output construction. 
Furthermore, compared to other paradigms associated with certain failure cases in constructing target output (Theorem~\ref{theorem-existing-paradigms-fails-to-construct-optimal-output}, Corollary~\ref{corollary-simple-sum-still-failures}), our 2-D graph convolution is capable to perform $0$ construction error, demonstrating its superiority against other paradigms.


\subsubsection{Irreducibility of parameter number in 2-D graph convolution}
\label{section-2dgraphconv-theoretical-investigation-of-2dgraphconv-irreducible-number-parameters}
We call the number of parameters in 2-D graph convolution irreducible for performing $0$ construction error towards target output, which is summarized in the following theorem: 
\begin{theorem}
\label{theorem-2dgraphconv-irreducible-number-parameters}
Let $\boldsymbol{\Psi}$ be the parameter set of 2-D graph convolution. 
For each subset $\boldsymbol{\Psi}_{\text{sub}}\subset\boldsymbol{\Psi}$, let $\mathcal{T}_{\boldsymbol{\Psi}_{\text{sub}}}$ be the 2-D graph convolution with partial parameters $\boldsymbol{\Psi}_{\text{sub}}$ (i.e., parameters in $\boldsymbol{\Psi}\setminus\boldsymbol{\Psi}_{\text{sub}}$ are constant or shared by $\boldsymbol{\Psi}_{\text{sub}}$). 
For any $\boldsymbol{\Psi}_{\text{sub}}$, there always exists certain cases of $\boldsymbol{F}$, $\boldsymbol{Z}^{*}$, making construction error $||\mathcal{T}_{\boldsymbol{\Psi}_{\text{sub}}}(\boldsymbol{F})-\boldsymbol{Z}^{*}||_{F}$ cannot achieve $0$ for all parameter values of $\boldsymbol{\Psi}_{\text{sub}}$. 
\end{theorem}
We prove this theorem in Appendix~\ref{appendix-proofs-and-derivations-theorem-2dgraphconv-irreducible-number-parameters}. 
Theorem~\ref{theorem-2dgraphconv-irreducible-number-parameters} proves that any special cases of the 2-D graph convolution, e.g., Paradigm {\bf (I)}, {\bf (II)} and {\bf (III)}, are inappropriate paradigm on target output construction. 
Therefore, despite involving more parameters than other paradigms, the proposed 2-D graph convolution is reasonable and necessary. 
On the other hand, since the paradigms with adding further parameters on 2-D graph convolution can also achieve $0$ construction error, the irreducibility of parameter number implies that our 2-D graph convolution is parameter-efficient compared to more complicated paradigms achieving $0$ error.


\section{ChebNet2D}
\label{section-chebnet2d}
To implement the 2-D graph convolution in practical scenarios, in this section we propose ChebNet2D, a spectral GNN performing efficient and effective construction of 2-D graph convolution with Chebyshev polynomial~\cite{poly_chebyshev}. 
We highlight each essential part of ChebNet2D to provide a comprehensive introduction to our method.  

\subsection{Filter Construction via Polynomial Approximation}
\label{section-chebnet2d-filter-construction-with-polynomial-approximation}
\par {\bf Chebyshev polynomial-based graph filter.} According to Eq.~\ref{equation-2dgraphconv}, implementation of 2-D graph convolution involves $NC^{2}$ number of learnable parameters, leading to critical issues in high model complexity. 
Inspired by ChebNet~\cite{ChebNet}, we perform efficient construction to the filter $\boldsymbol{g}^{(c,j)}$ with the same Chebyshev polynomial approximation~\cite{poly_chebyshev}, making the filter construction be $\boldsymbol{g}^{(c,j)}\doteq\sum_{d=0}^{D}\theta_{d}^{(c,j)}T_{d}(\boldsymbol{\lambda})$, where $T_{d}(\cdot)$ is the $d$-th order Chebyshev polynomial with $\theta_{d}^{(c,j)}$ as learnable coefficient, and $D$ is the degree of truncated polynomial. 
Accordingly, the convolution operator $\boldsymbol{\Phi}^{(c,j)}_{\mathcal{G}}$ is formulated as follows: 
\begin{align}
\label{equation-chebyshev-approximation}
\boldsymbol{\Phi}^{(c,j)}_{\mathcal{G}}&\doteq\sum_{d=0}^{D}\theta_{d}^{(c,j)}T_{d}(\Hat{\boldsymbol{L}}), 
\end{align}
where $\Hat{\boldsymbol{L}}=\boldsymbol{L}-\boldsymbol{I}$ is set to match the convergence domain of Chebyshev polynomial~\cite{polyapprox_1,polyapprox_2,poly_chebyshev}. 
With the polynomial approximation above, the parameter number of 2-D graph convolution is reduced to $(D+1)C^{2}$ far less than the original $NC^{2}$, making significant improvement on parameter efficiency. 
Moreover, Chebyshev polynomial maintains \textbf{orthogonality} on polynomial basis~\cite{orthogonalpoly_1_poly_jacobian,orthogonalpoly_2}, guaranteeing superior convergence rate on filter construction~\cite{JacobiConv,OptBasisGNN}. 
\par {\bf High quality construction with Chebyshev interpolation.} As discussed in~\cite{ChebNetII}, constructing graph filters with high degree Chebyshev polynomial may encounter Runge Phenomenon~\cite{RungePhenomenon}, which leads to a detrimental effect on the construction quality and can be alleviated with involving Chebyshev interpolation~\cite{chebyshevinterpolation} defined as follows:
\begin{definition}
\label{definition-chebyshev-interpolation}
({\bf Chebyshev interpolation})~\cite{ChebNetII} Given a continuous filter function $g(x)$ defined on $(-1,1)$, let $x_{b}=\cos(\frac{b+1/2}{K+1}\pi)$, $b=0,1,...,D$ denote the Chebyshev nodes for $T_{D+1}(x)$, i.e., the zeros of $T_{D+1}(x)$, and $g(x_{b})$ denotes the function value at $x_{b}$. 
The Chebyshev interpolation of $g(x)$ is defined as below:
\begin{equation}
\label{equation-chebyshev-interpolation}
g(x)=\sum_{d=0}^{D}\theta_{d}^{'}T_{d}(x),\ \theta_{d}=\frac{2}{D+1}\sum_{b=0}^{D}g(x_{b})T_{d}(x_{b}),
\end{equation}
where the prime indicates the first term is to be halved, i.e., $\theta_{0}^{'}=\theta_{0}/2$, $\theta_{1}^{'}=\theta_{1}$,...,$\theta_{D}^{'}=\theta_{D}$. 
\end{definition}
Thus, similar to~\cite{ChebNetII}, through replacing $g(x_{b})$ in Eq.~\ref{equation-chebyshev-interpolation} with learnable parameter $\theta_{b}^{(c,j)}$, the construction of $\boldsymbol{\Phi}^{(c,j)}_{\mathcal{G}}$ with Chebyshev interpolation is formulated as follows: 
\begin{align}
\label{equation-chebyshev-interpolation-approximation}
\boldsymbol{\Phi}^{(c,j)}_{\mathcal{G}}&\doteq\frac{2}{D+1}\sum_{d=0}^{D}\sum_{b=0}^{D}\theta_{b}^{(c,j)}T_{d}(x_{b})T_{d}(\Hat{\boldsymbol{L}}). 
\end{align}
\par {\bf Streamlined paradigm of 2-D graph convolution.} Though being simplified with polynomial approximation, Eq.~\ref{equation-2dgraphconv} involves complicated vectorization operation and large size operator matrix, leading to critical complexity issues on practical implementation. 
To achieve easy-to-implement framework, based on Eq.~\ref{equation-chebyshev-interpolation-approximation}, we reformulate a streamlined formulation of Eq.~\ref{equation-2dgraphconv} as follows: 
\begin{align}
\label{equation-chebyshev-polynomial-2dgraphconv}
\boldsymbol{Z}=\sum_{d=0}^{D}T_{d}(\Hat{\boldsymbol{L}})\boldsymbol{F}\left(\sum_{b=0}^{D}T_{d}(x_{b})\boldsymbol{\Theta}_{::b}\right).  
\end{align}
Here, the $3$D tensor $\boldsymbol{\Theta}\in\mathbb{R}^{C\times C\times (D+1)}$ is named \textit{parameterized coefficient tensor} with $\boldsymbol{\Theta}_{cjb}=\theta_{b}^{(c,j)}$. 
We defer the derivation of obtaining Eq.~\ref{equation-chebyshev-polynomial-2dgraphconv} based on Eq.~\ref{equation-2dgraphconv} and Eq.~\ref{equation-chebyshev-interpolation-approximation} in Appendix~\ref{appendix-proofs-and-derivations-obtaining-chebyshev-polynomial-2dgraphconv}. 
Since $T_{d}(x_{b})$ can be precomputed and $\sum_{b=0}^{D}T_{d}(x_{b})\boldsymbol{\Theta}_{::b}$ involves only the sum of matrices, Eq.~\ref{equation-chebyshev-polynomial-2dgraphconv} provides a significantly streamlined paradigm for the implementation of 2-D graph convolution.

\subsection{Overall Model of ChebNet2D}
\label{section-chebnet2d-overall-model}
Based on Section~\ref{section-chebnet2d-filter-construction-with-polynomial-approximation}, we introduce the model paradigm of our ChebNet2D. 
Specifically, given input graph $\mathcal{G}$ with node feature $\boldsymbol{X}$ and normalized graph Laplacian $\boldsymbol{L}$, the ChebNet2D generates output $\boldsymbol{Z}$ as follows: 
\begin{align}
\label{equation-paradigm-chebnet2d}
\boldsymbol{Z}=\sum_{d=0}^{D}T_{d}(\Hat{\boldsymbol{L}})h_{\eta}(\boldsymbol{X})\left(\sum_{b=0}^{D}T_{d}(x_{b})\boldsymbol{\Theta}_{::b}\right)\ ,
\end{align}
where $h_{\eta}$ denotes an MLP with $\eta$ as parameters. 
With taking $\boldsymbol{F}=h_{\eta}(\boldsymbol{X})$, the ChebNet2D decouples feature transformation and spectral graph convolution, making the model consistent to our theoretical analysis. 
Similar decoupling paradigm has been widely used in recent works and has achieved leading performance~\cite{APPNP,GPRGNN,BernNet-GNN-narrowbandresults-1,ChebNetII,OptBasisGNN}, thereby guaranteeing the effectiveness of such paradigm. 
\par {\bf Complexity analysis.} The computation complexity and parameter complexity of ChebNet2D are $\mathcal{O}(NH(F+C)+DN^{2}C+(D+1)NC^{2})$ and $\mathcal{O}(H(F+C)+(D+1)C^{2})$ respectively, where $H$ denotes the hidden layer dimension of MLP, $C$ is the output dimension, and $D$ is the order of polynomial. 
While the spectral GNNs with similar decoupling paradigm~\cite{APPNP,GPRGNN,BernNet-GNN-narrowbandresults-1,ChebNetII,OptBasisGNN} are less complicated with $\mathcal{O}(NH(F+C)+DN^{2}C)$ on computation complexity and $\mathcal{O}(H(F+C))$ on parameter complexity, the ChebNet2D is still comparable in efficiency. 
Specifically, as both $N$ and $F$ are significantly larger than $C$, $D$ in practical scenarios, compared to the counterparts, the complexity increases on the ChebNet2D, i.e., $(D+1)NC^{2}$ and $(D+1)C^{2}$, are trivial against the major terms $NH(F+C)+DN^{2}C$ and $H(F+C)$. 
In the following section, by showing the training time of ChebNet2D and its counterparts, we verify the competitive efficiency of ChebNet2D.


\section{Related Works}
\label{section-related-works}
\par {\bf Spectral GNNs.} Spectral GNNs are GNNs that perform essential spectral graph convolution on input graph data (signals). 
Previous research focused on designing various spectral graph filters used for the convolution operation, which are categorized into two classes following~\cite{BernNet-GNN-narrowbandresults-1,JacobiConv}: 
\begin{itemize}[leftmargin=*,parsep=2pt,itemsep=2pt,topsep=2pt]
\item \textit{Spectral GNNs with fixed filters}: This class of spectral GNNs are constructed by involving pre-defined and fixed spectral graph filters, leading to significantly simple and easy-to-implement model architectures. 
For instance, APPNP~\cite{APPNP} utilizes Personalized PageRank (PPR)~\cite{pagerank} to build filter functions. 
ADC~\cite{ADC} involves pre-defined convolution (diffusion) operators and constructs hidden representation with modifying variable propagation coefficients. 
GNN-HF/LF~\cite{GNN-HF-LF} constructs filter weights from the perspective of graph optimization functions, which simulates both high- and low-pass filters. 
\item \textit{Spectral GNNs with learnable filters:} This class of spectral GNNs perform convolution with learnable filters adaptive to datasets, leading to flexible models and better task performance. 
For instance, GPRGNN~\cite{GPRGNN} learns a polynomial filter by directly performing gradient descent on the polynomial coefficients. 
DSGC~\cite{DSGC-spectral-2dgnn-2dGFT-gnn-2} performs both row-wise and column-wise graph convolution based on multi-dimensional GFT~\cite{2dGFT}. 
BernNet~\cite{BernNet-GNN-narrowbandresults-1} expresses the filtering operation with Bernstein polynomials. 
ChebNetII~\cite{ChebNetII} improves ChebNet~\cite{ChebNet} with applying Chebyshev interpolation~\cite{chebyshevinterpolation}. 
JacobiConv~\cite{JacobiConv} removes non-linearity and applies Jacobian polynomial~\cite{orthogonalpoly_1_poly_jacobian} to perform column-wise convolution. 
OptBasis~\cite{OptBasisGNN} learns optimal polynomial basis specific to each dataset with Favard's theory, leading to state-of-the-art GNN. 
\end{itemize}
\par {\bf Non-spectral GNNs.} Non-spectral GNNs, also called spatial GNNs, are GNNs designed based on the principle of message-passing neural network (MPNN)~\cite{MessagePNN}. 
Previous works concentrated on performing message aggregation and propagation associated with the spatial structure of graphs~\cite{comprehensivegnn}. 
For instance, GCN~\cite{GCN} are constructed through stacking $1$-order ChebNet layer~\cite{ChebNet} and non-linearity functions. 
GCNII~\cite{gcnii} employs residual connection and identity mapping to GCN, leading to substantial deep-layer GNN with promising performance. 
LINKX~\cite{dataset6-large-hetero} is a simple MLP-based model that embeds both graph structure (adjacency) and node feature to generate output. 
Nodeformer~\cite{nodeformer} propagates messages between arbitrary node pairs in layer-specific latent graphs, which is the first Transformer model that scales all-pair message passing to large graphs. 
GloGNN++~\cite{glognn++} generates a node’s embedding by aggregating information from global nodes in the graph. 


\section{Empirical Studies}
\label{section-empirical-studies}
In this section, by conducting node classification tasks on $18$ benchmark graph datasets, we evaluate the performance of ChebNet2D against state-of-the-art GNNs. 

\subsection{Node Classification}
\label{section-empirical-studies-node-classification}
\par {\bf Datasets.} We use $5$ homophilic datasets, including citation graphs: Cora, CiteSeer, and PubMed~\cite{dataset1-cora}, and the Amazon co-purchase graphs: Computers and Photo~\cite{dataset2-photo-comp}. 
We also involve $5$ heterophilic datasets, including Wikipedia graphs: Chameleon and Squirrel~\cite{dataset4-cham-squi}, co-occurrence graph: Actor~\cite{dataset7-actor}, and webpage graphs from WebKB\footnote{\url{http://www.cs.cmu.edu/afs/cs.cmu.edu/project/theo-11/www/wwkb}}: Texas and Cornell~\cite{dataset3-pei}. 
For data split on these datasets, we take conventional $60\%/20\%/20\%$ train/validation/test split ratio following~\cite{GPRGNN,BernNet-GNN-narrowbandresults-1,ChebNetII,JacobiConv,OptBasisGNN}. 
We generate $10$ random splits for each dataset and evaluate all models on the same splits, where each model is evaluated $10$ times with $10$ random initializations on each random split. 
\par {\bf Baselines.} We take $10$ spectral GNNs associated with different spectral graph convolution paradigms: for Paradigm {\bf (I)}, we include APPNP~\cite{APPNP}, GPRGNN~\cite{GPRGNN}, GNN-HF/LF~\cite{GNN-HF-LF}, BernNet~\cite{BernNet-GNN-narrowbandresults-1}, and state-of-the-art methods ChebNetII~\cite{ChebNetII} and OptBasis~\cite{OptBasisGNN}; for Paradigm {\bf (II)}, we include DSGC~\cite{DSGC-spectral-2dgnn-2dGFT-gnn-2} and SOTA method Spec-GN~\cite{specGN}; for Paradigm {\bf (III)}, we include ADC~\cite{ADC} and SOTA method JacobiConv~\cite{JacobiConv}. 
We further include $5$ non-spectral graph learning methods, including GCN~\cite{GCN}, GCNII~\cite{gcnii}, PDE-GCN~\cite{pdegcn} and SOTA methods Nodeformer~\cite{nodeformer} and GloGNN++~\cite{glognn++}. 
More details about datasets and baselines can be found in Appendix~\ref{appendix-experimental-setup}. 

\begin{table*}[!ht]
  \caption{Results on $10$ medium-sized datasets: Mean accuracy $(\%)$ $\pm$ standard deviation. 
  Baselines are classified into two types, i.e., spectral GNNs (denoted as {\bf spec}) and Non-spectral methods (denoted as {\bf non-spec}).} 
  \vskip 0.15in
  \label{table-result-middle}
  \centering
  \setlength{\tabcolsep}{2pt}
  \renewcommand\arraystretch{1.2}
  \resizebox{\textwidth}{!}{
  \begin{tabular}{lcccccc|ccccc}
    \hline  
    \multirow{2}{*}{Type} & \multirow{2}{*}{Method} 
    & \multicolumn{5}{c|}{Homophilic} & \multicolumn{5}{c}{Heterophilic} \\ \cline{3-12}
    & &   Cora  &  Cite.  &  Pubm.  &  Comp.  &  Photo  &  Cham. &  Squi. &  Texas &  Corn. &  Actor \\ \hline  
    \multirow{10}{*}{{\bf spec}} 
    &  APPNP & $88.94_{\pm1.1}$ & $77.38_{\pm1.5}$ & $88.76_{\pm0.6}$ & $90.22_{\pm0.7}$ & $91.37_{\pm0.7}$ & $53.83_{\pm1.2}$ & $36.42_{\pm1.8}$ & $75.81_{\pm2.9}$ & $72.21_{\pm3.6}$ & $37.66_{\pm1.0}$ \\ 
    &  GPRGNN & $\underline{89.41_{\pm0.7}}$ & $77.73_{\pm1.1}$ & $89.11_{\pm0.4}$ & $90.92_{\pm0.6}$ & $94.27_{\pm0.5}$ & $70.73_{\pm1.4}$ & $55.62_{\pm1.2}$ & $85.03_{\pm3.2}$ & $\underline{85.79_{\pm4.0}}$ & $40.16_{\pm0.8}$ \\ 
    &  GNN-HF/LF & $88.91_{\pm0.8}$ & $77.56_{\pm0.8}$ & $89.07_{\pm0.5}$ & $90.74_{\pm0.3}$ & $94.42_{\pm0.3}$ & $67.81_{\pm1.1}$ & $51.38_{\pm2.9}$ & $81.78_{\pm5.1}$ & $80.48_{\pm3.8}$ & $\underline{41.17_{\pm0.7}}$ \\ 
    &  BernNet & $89.22_{\pm1.5}$ & $77.37_{\pm1.3}$ & $89.42_{\pm0.8}$ & $91.35_{\pm0.6}$ & $95.19_{\pm0.5}$ & $71.66_{\pm1.5}$ & $57.47_{\pm2.0}$ & $82.42_{\pm6.4}$ & $79.35_{\pm5.9}$ & $40.57_{\pm0.9}$ \\ 
    &  ChebNetII & $89.32_{\pm1.1}$ & $77.61_{\pm1.0}$ & $89.26_{\pm0.5}$ & $91.53_{\pm0.8}$ & $95.33_{\pm0.4}$ & $73.48_{\pm1.1}$ & $61.73_{\pm1.5}$ & $\underline{85.81_{\pm4.5}}$ & $\boldsymbol{86.78_{\pm5.3}}$ & $41.03_{\pm1.2}$ \\ 
    &  OptBasis & $89.33_{\pm0.7}$ & $\underline{77.87_{\pm0.9}}$ & $89.82_{\pm0.8}$ & $91.63_{\pm0.4}$ & $95.31_{\pm0.8}$ & $73.88_{\pm1.3}$ & $\underline{63.22_{\pm1.6}}$ & $83.63_{\pm5.0}$ & $84.23_{\pm5.2}$ & $41.11_{\pm0.7}$ \\ 
    &  DSGC & $88.51_{\pm1.3}$ & $76.48_{\pm0.7}$ & $88.45_{\pm0.7}$ & $90.27_{\pm0.6}$ & $93.05_{\pm0.6}$ & $57.38_{\pm1.9}$ & $43.77_{\pm2.1}$ & $79.37_{\pm2.9}$ & $80.12_{\pm4.4}$ & $39.77_{\pm1.3}$ \\ 
    &  Spec-GN & $88.38_{\pm1.1}$ & $77.11_{\pm1.2}$ & $88.93_{\pm0.4}$ & $90.59_{\pm0.3}$ & $94.57_{\pm0.5}$ & $66.63_{\pm1.4}$ & $51.38_{\pm1.8}$ & $73.19_{\pm6.3}$ & $81.68_{\pm6.7}$ & $40.47_{\pm0.8}$ \\ 
    &  ADC & $88.43_{\pm0.7}$ & $76.51_{\pm1.1}$ & $88.62_{\pm0.7}$ & $89.56_{\pm0.5}$ & $90.39_{\pm0.5}$ & $62.71_{\pm1.5}$ & $50.88_{\pm1.3}$ & $76.74_{\pm3.5}$ & $77.50_{\pm2.9}$ & $39.24_{\pm0.8}$ \\ 
    &  JacobiConv & $88.95_{\pm0.6}$ & $77.43_{\pm1.2}$ & $89.53_{\pm0.4}$ & $91.48_{\pm0.4}$ & $\underline{95.41_{\pm0.5}}$ & $\underline{74.19_{\pm1.1}}$ & $58.87_{\pm1.2}$ & $85.48_{\pm3.7}$ & $85.56_{\pm4.9}$ & $40.64_{\pm1.1}$ \\ \hline
    \multirow{5}{*}{\bf \makecell{non- \\ spec}} 
    &  GCN & $86.68_{\pm1.6}$ & $74.89_{\pm1.1}$ & $86.85_{\pm0.8}$ & $88.77_{\pm0.8}$ & $90.56_{\pm0.5}$ & $63.27_{\pm2.2}$ & $43.88_{\pm2.9}$ & $62.69_{\pm10.8}$ & $58.31_{\pm12.8}$ & $35.12_{\pm1.6}$ \\ 
    &  GCNII & $88.51_{\pm1.2}$ & $77.15_{\pm1.2}$ & $89.23_{\pm1.5}$ & $89.72_{\pm1.0}$ & $92.14_{\pm0.8}$ & $62.38_{\pm2.5}$ & $44.53_{\pm2.2}$ & $75.47_{\pm7.5}$ & $69.81_{\pm8.9}$ & $37.81_{\pm0.8}$ \\ 
    &  PDE-GCN & $88.79_{\pm1.1}$ & $77.67_{\pm1.4}$ & $89.21_{\pm0.7}$ & $90.13_{\pm0.7}$ & $91.78_{\pm0.6}$ & $67.61_{\pm1.2}$ & $47.81_{\pm1.5}$ & $\boldsymbol{86.92_{\pm3.8}}$ & $79.83_{\pm4.6}$ & $38.46_{\pm0.9}$ \\ 
    & Nodeformer & $87.73_{\pm1.3}$ & $76.69_{\pm1.1}$ & $\underline{90.03_{\pm0.5}}$ & $\underline{91.67_{\pm0.6}}$ & $95.26_{\pm0.3}$ & $55.87_{\pm1.0}$ & $44.71_{\pm0.8}$ & $82.48_{\pm4.7}$ & $80.63_{\pm5.1}$ & $38.81_{\pm0.7}$ \\ 
    &  GloGNN++ & $88.53_{\pm1.2}$ & $77.02_{\pm1.1}$ & $89.38_{\pm0.4}$ & $90.67_{\pm0.5}$ & $93.67_{\pm0.8}$ & $73.55_{\pm1.3}$ & $60.93_{\pm2.0}$ & $79.72_{\pm4.1}$ & $83.13_{\pm3.9}$ & $39.29_{\pm1.2}$ \\ \hline
    {\bf ours} &  ChebNet2D & $\boldsymbol{91.24_{\pm0.9}}$ & $\boldsymbol{79.23_{\pm1.2}}$ & $\boldsymbol{91.86_{\pm0.6}}$ & $\boldsymbol{93.19_{\pm0.5}}$ & $\boldsymbol{96.83_{\pm0.6}}$ & $\boldsymbol{76.08_{\pm1.3}}$ & $\boldsymbol{65.23_{\pm1.3}}$ & $85.39_{\pm4.3}$ & $84.97_{\pm5.1}$ & $\boldsymbol{42.62_{\pm0.8}}$ \\
    \hline
  \end{tabular}}
\end{table*}
\par {\bf Results.} We report experimental results in Table~\ref{table-result-middle}, where \textbf{boldface} and \underline{underline} denote the best and second result respectively. 
We observe that ChebNet2D outperforms all baselines on $8$ datasets by a noticeable margin, and achieves competitive results on the rest $2$ datasets, i.e., Taxas and Cornell. 
Note that previous research~\cite{dataset8-small-hetero} has pointed out that both Texas and Cornell are too small in size with highly imbalanced classes, which leads to unstable results and misleading evaluations of the baselines. 
Thus, based on the results of the rest $8$ representative datasets, we empirically verify the superior effectiveness of our ChebNet2D.

\subsection{Node Classification on Challenging Datasets}
\label{section-empirical-studies-node-classification-large-challenging}
\par {\bf Datasets.} We further include both large-scale and latest heterophilic datasets into experiments, including large-scale homophilic graphs: ogbn-arxiv and ogbn-products~\cite{dataset5-ogb}, large-scale heterophilic graphs: Penn94, Genius and Gamers~\cite{dataset6-large-hetero}, latest heterophilic graphs: Roman-empire, Tolokers and Amazon-ratings~\cite{dataset8-small-hetero}. 
For large-scale graphs, we use the split provided by the original papers, where each model is evaluated $10$ times on each dataset. 
For the latest heterophilic graphs, we generate $10$ random splits for each dataset with $50\%/25\%/25\%$ train/test/valid ratio as suggested in~\cite{dataset8-small-hetero}, where each model is evaluated $10$ times on each split. 
\par {\bf Baselines.} We take GCN and other $9$ state-of-the-art GNNs as baselines, including APPNP, GPRGNN, Spec-GN, JacobiConv, ChebNetII, OptBasis, LINKX~\cite{dataset6-large-hetero}, Nodeformer and GloGNN++. 
More details about datasets and baselines can be found in Appendix~\ref{appendix-experimental-setup}. 

\begin{table*}[!ht]
  \caption{Results on large-scale and latest heterophilic datasets: Mean accuracy $(\%)$ $\pm$ standard deviation.} 
  \vskip 0.15in
  \label{table-result-large-challenging}
  \centering
  \setlength{\tabcolsep}{6pt}
  \renewcommand\arraystretch{1.1}
  \resizebox{\textwidth}{!}{
  \begin{tabular}{lcc|ccc|ccc}
    \hline  
    \multirow{2}{*}{Method} 
    & \multicolumn{2}{c|}{Large homophilic} & \multicolumn{3}{c|}{Large heterophilic} & \multicolumn{3}{c}{Latest heterophilic} \\ \cline{2-9}
    & -arxiv  &  -products  &  Penn94  &  Genius  &  Gamers  &  Roman-empire &  Tolokers  &  Amazon-ratings \\ \hline  
    GCN  & $71.74_{\pm0.3}$ & $75.64_{\pm0.2}$ & $82.47_{\pm0.3}$ & $87.42_{\pm0.3}$ & $62.18_{\pm0.3}$ & $55.78_{\pm0.4}$ & $76.65_{\pm0.8}$ & $47.38_{\pm0.4}$ \\ 
    APPNP  & $71.19_{\pm0.5}$ & $78.92_{\pm0.4}$ & $77.68_{\pm0.4}$ & $87.45_{\pm0.6}$ & $63.34_{\pm0.3}$ & $72.74_{\pm0.6}$ & $73.92_{\pm0.6}$ & $48.41_{\pm0.3}$ \\ 
    GPRGNN  & $71.96_{\pm0.4}$ & $80.75_{\pm0.3}$ & $83.87_{\pm0.6}$ & $90.07_{\pm0.4}$ & $63.18_{\pm0.6}$ & $73.44_{\pm0.3}$ & $75.12_{\pm0.6}$ & $49.56_{\pm0.4}$ \\ 
    Spec-GN  & $71.73_{\pm0.2}$ & $79.67_{\pm0.3}$ & $84.18_{\pm0.5}$ & $89.69_{\pm0.5}$ & $62.93_{\pm0.3}$ & $71.29_{\pm0.6}$ & $78.55_{\pm0.7}$ & $46.58_{\pm0.6}$ \\ 
    JacobiConv  & $72.18_{\pm0.2}$ & $81.31_{\pm0.4}$ & $84.56_{\pm0.4}$ & $90.63_{\pm0.3}$ & $65.23_{\pm0.3}$ & $74.35_{\pm0.6}$ & $78.41_{\pm0.3}$ & $48.56_{\pm0.2}$ \\ 
    ChebNetII  & $\underline{72.33_{\pm0.3}}$ & $81.07_{\pm0.3}$ & $85.22_{\pm0.3}$ & $90.47_{\pm0.2}$ & $65.71_{\pm0.2}$ & $\underline{74.58_{\pm0.3}}$ & $79.02_{\pm0.6}$ & $49.43_{\pm0.4}$ \\ 
    OptBasis  & $72.26_{\pm0.4}$ & $81.15_{\pm0.3}$ & $84.82_{\pm0.6}$ & $\underline{90.91_{\pm0.2}}$ & $65.86_{\pm0.4}$ & $74.29_{\pm0.3}$ & $79.25_{\pm0.5}$ & $49.48_{\pm0.3}$ \\ 
    LINKX  & $67.23_{\pm0.3}$ & $77.52_{\pm0.7}$ & $85.15_{\pm0.2}$ & $90.69_{\pm0.4}$ & $66.48_{\pm0.3}$ & $64.81_{\pm0.8}$ & $78.11_{\pm0.7}$ & $\boldsymbol{51.13_{\pm0.6}}$ \\ 
    Nodeformer  & $65.08_{\pm0.6}$ & $73.66_{\pm0.5}$ & $83.71_{\pm0.3}$ & $89.38_{\pm0.5}$ & $63.18_{\pm0.6}$ & $73.42_{\pm0.6}$ & $\underline{79.47_{\pm0.4}}$ & $45.48_{\pm0.5}$ \\
    GloGNN++  & $72.11_{\pm0.5}$ & $\underline{81.51_{\pm0.4}}$ & $\underline{85.54_{\pm0.3}}$ & $90.72_{\pm0.2}$ & $\underline{66.52_{\pm0.4}}$ & $67.39_{\pm0.3}$ & $79.08_{\pm0.7}$ & $49.88_{\pm0.4}$ \\ \hline
    ChebNet2D & $\boldsymbol{72.87_{\pm0.2}}$ & $\boldsymbol{82.88_{\pm0.4}}$ & $\boldsymbol{87.62_{\pm0.2}}$ & $\boldsymbol{91.88_{\pm0.3}}$ & $\boldsymbol{67.65_{\pm0.2}}$ & $\boldsymbol{77.18_{\pm0.5}}$ & $\boldsymbol{82.66_{\pm0.4}}$ & $\underline{50.77_{\pm0.3}}$ \\
    \hline
  \end{tabular}}
  \vskip -0.05in
\end{table*}
\par {\bf Results.} We report experimental results in Table~\ref{table-result-large-challenging}. 
We observe that ChebNet2D achieves significant performance on all $8$ datasets with $7$ best and $1$ second accuracy, demonstrating the substantial scalability of our method. 
Notably, other methods achieve promising performance on only part of these datasets but perform averagely on the rest datasets. 
In contrast, our ChebNet2D maintains leading results on all datasets with noticeable accuracy improvements over the second-best, indicating the effectiveness of our approach.

\subsection{Comparison on Convolution Paradigms}
\label{section-empirical-studies-ablation-study-comparison-on-convolution-paradigms}
We involve more intuitive ablation experiments to show the superiority of 2-D graph convolution against other paradigms. 
Specifically, we re-implement the baselines of different convolution paradigms into a unified decoupling model architecture as Eq.~\ref{equation-paradigm-chebnet2d}, where an MLP is used for feature transformation. 
Thus, the comparison between those reconstructed baselines is equivalent to the fair comparison of convolution paradigms. 

\begin{table*}[!ht]
  \caption{Experimental results on $8$ representative datasets: Mean accuracy $(\%)$ $\pm$ standard deviation. 
  Spec-GN$^{*}$ and JacobiConv$^{*}$ denote these baselines are re-implemented with unified decoupling model architecture for fairness.} 
  \vskip 0.15in
  \label{table-result-ablation}
  \centering
  \setlength{\tabcolsep}{5.5pt}
  \renewcommand\arraystretch{1.1}
  \resizebox{\textwidth}{!}{
  \begin{tabular}{lcccccccc}
    \hline  
    Method  &   Cora  &  Comp.  &  -arxiv  &  -products  &  Penn94  &  Gamers &  Roman-empire  &  Amazon-ratings  \\ \hline  
    GPRGNN & $\underline{89.41_{\pm0.7}}$ & $90.92_{\pm0.6}$ & $71.96_{\pm0.4}$ & $80.75_{\pm0.3}$ & $83.87_{\pm0.6}$ & $63.18_{\pm0.6}$ & $73.44_{\pm0.3}$ & $\underline{49.56_{\pm0.4}}$ \\ 
    ChebNetII & $89.32_{\pm1.1}$ & $91.53_{\pm0.8}$ & $72.33_{\pm0.3}$ & $81.07_{\pm0.3}$ & $\underline{85.22_{\pm0.3}}$ & $65.71_{\pm0.2}$ & $\underline{74.58_{\pm0.3}}$ & $49.43_{\pm0.4}$ \\ 
    OptBasis & $89.33_{\pm0.7}$ & $\underline{91.63_{\pm0.4}}$ & $72.26_{\pm0.4}$ & $81.15_{\pm0.3}$ & $84.82_{\pm0.6}$ & $\underline{65.86_{\pm0.4}}$ & $74.29_{\pm0.3}$ &  $49.48_{\pm0.3}$  \\ 
    Spec-GN$^{*}$ & $88.98_{\pm0.8}$ & $91.15_{\pm0.4}$ & $71.97_{\pm0.4}$ & $79.25_{\pm0.5}$ & $84.43_{\pm0.4}$ & $63.62_{\pm0.5}$ & $72.77_{\pm0.4}$ & $48.62_{\pm0.8}$ \\ 
    JacobiConv$^{*}$ & $89.41_{\pm0.8}$ & $91.52_{\pm0.4}$ & $\underline{72.35_{\pm0.3}}$ & $\underline{81.53_{\pm0.3}}$ & $85.18_{\pm0.5}$ & $65.49_{\pm0.4}$ & $74.12_{\pm0.4}$ & $47.79_{\pm0.3}$  \\ \hline
    ChebNet2D & $\boldsymbol{91.24_{\pm0.9}}$ & $\boldsymbol{93.19_{\pm0.5}}$ & $\boldsymbol{72.87_{\pm0.2}}$ & $\boldsymbol{82.88_{\pm0.4}}$ & $\boldsymbol{87.62_{\pm0.2}}$ & $\boldsymbol{67.65_{\pm0.2}}$ & $\boldsymbol{77.18_{\pm0.5}}$ & $\boldsymbol{50.77_{\pm0.3}}$ \\
    \hline
  \end{tabular}}
  \vskip -0.05in
\end{table*}
\par The experimental results are shown in Table~\ref{table-result-ablation}. 
According to the performance comparison, in the same decoupled model architecture, ChebNet2D outperforms other baselines in all datasets by a significant margin, which demonstrates the superiority of our 2-D graph convolution against other convolution paradigms.

\subsection{Efficiency Evaluation}
\label{section-empirical-studies-efficiency-evaluation}
To fairly evaluate the model efficiency of our ChebNet2D, we report the model running time and parameter number of ChebNet2D on large-scale datasets and compare them to two state-of-the-art spectral GNNs, i.e., ChebNetII and OptBasis. 
Both ChebNetII and OptBasis adapt the same decoupling model structure with MLP as feature transformation, where the same Chebyshev interpolation is applied in ChebNetII, thereby allowing for an intuitive comparison on model efficiency. 

\begin{table}[!ht]
  \vskip -0.15in
  \caption{Average running time per epoch / average total running time. 
  For -arxiv and -products, the data is \textbf{s} / \textbf{min}; for the rest datasets, the data is \textbf{ms} / \textbf{s}.} 
  \vskip 0.15in
  \label{table-result-efficiency-evaluation}
  \centering
  \setlength{\tabcolsep}{10pt}
  \resizebox{.47\textwidth}{!}{
  \begin{tabular}{lcc|c}
    \hline  
    Dataset  & OptBasis  &  ChebNetII  &  ChebNet2D  \\ \hline  
    -arxiv &  $1.4$ / $14.1$ &  $1.3$ / $13.8$ &  $1.5$ / $14.5$ \\ 
    -products & $8.3$ / $75.9$ &  $7.9$ / $71.2$ &  $8.3$ / $74.5$ \\ 
    Penn94 &  $48.7$ / $30.2$ &  $46.6$ / $30.8$ &  $48.9$ / $30.8$ \\ 
    Genius &  $49.7$ / $27.7$ &  $49.7$ / $28.1$ &  $51.3$ / $28.1$ \\ 
    Gamers &  $57.4$ / $31.0$ &  $54.8$ / $30.5$ &  $57.2$ / $30.4$ \\ 
    \hline
  \end{tabular}}
\end{table}
\par As shown in Table~\ref{table-result-efficiency-evaluation}, ChebNet2D achieves competitive efficiency on all large-scale datasets against OptBasis and ChebNetII, and even outperforms them in average total running time on the Gamers dataset. 
Furthermore, according to our analysis in Section~\ref{section-chebnet2d-overall-model}, both OptBasis and ChebNetII are with the simplest Paradigm {\bf (I)} convolution that involves less computation complexity than our ChebNet2D. 
Thus, although ChebNet2D involves more parameters, the practical complexity increase of ChebNet2D over other methods is trivial, indicating the noticeable efficiency of our proposal.


\section{Conclusion}
\label{section-conclusion}
In this paper, we first revisit spectral graph convolution paradigms of existing spectral GNNs. 
With regarding convolution operation as construction to target output, we prove that existing popular convolution paradigms cannot construct the target output with mild conditions on input graph signals, which causes spectral GNNs to fall into suboptimal solutions. 
To address issues of existing convolution paradigms, we propose a new convolution paradigm named 2-D graph convolution. 
The proposed 2-D graph convolution unifies previous convolution paradigms, and is provably capable to construct arbitrary output with significant parameter efficiency. 
Based on 2-D graph convolution, we propose ChebNet2D, an effective and efficient spectral GNN implementation of 2-D graph convolution. 
ChebNet2D outperforms previous SOTA spectral GNNs on $18$ datasets, which verifies the superiority of our 2-D graph convolution against other convolution paradigms.

\section*{Impact Statements}
This paper presents work whose goal is to advance the field of Machine Learning. There are many potential societal consequences of our work, none of which we feel must be specifically highlighted here.

\bibliography{example_paper}

\begin{thebibliography}{73}
\providecommand{\natexlab}[1]{#1}
\providecommand{\url}[1]{\texttt{#1}}
\expandafter\ifx\csname urlstyle\endcsname\relax
  \providecommand{\doi}[1]{doi: #1}\else
  \providecommand{\doi}{doi: \begingroup \urlstyle{rm}\Url}\fi

\bibitem[An et~al.(2022)An, Deng, Guo, Feng, Zhu, Yang, and Liu]{fully-connected-6}
An, X., Deng, J., Guo, J., Feng, Z., Zhu, X., Yang, J., and Liu, T.
\newblock Killing two birds with one stone: Efficient and robust training of face recognition cnns by partial fc.
\newblock In \emph{Proceedings of the IEEE/CVF Conference on Computer Vision and Pattern Recognition (CVPR)}, pp.\  4042--4051, June 2022.

\bibitem[Baksalary \& Kala(1979)Baksalary and Kala]{generalized-sylvester-equation-2}
Baksalary, J. and Kala, R.
\newblock The matrix equation ax- yb= c.
\newblock \emph{Linear Algebra and its Applications}, 25:\penalty0 41--43, 1979.

\bibitem[Cao et~al.(2020)Cao, Wang, Duan, Zhang, Zhu, Huang, Tong, Xu, Bai, Tong, and Zhang]{spectralGNN-3}
Cao, D., Wang, Y., Duan, J., Zhang, C., Zhu, X., Huang, C., Tong, Y., Xu, B., Bai, J., Tong, J., and Zhang, Q.
\newblock Spectral temporal graph neural network for multivariate time-series forecasting.
\newblock In Larochelle, H., Ranzato, M., Hadsell, R., Balcan, M., and Lin, H. (eds.), \emph{Advances in Neural Information Processing Systems}, volume~33, pp.\  17766--17778. Curran Associates, Inc., 2020.
\newblock URL \url{https://proceedings.neurips.cc/paper_files/paper/2020/file/cdf6581cb7aca4b7e19ef136c6e601a5-Paper.pdf}.

\bibitem[Cao et~al.(2021)Cao, Li, Ma, and Tomizuka]{spectralGNN-4}
Cao, D., Li, J., Ma, H., and Tomizuka, M.
\newblock Spectral temporal graph neural network for trajectory prediction.
\newblock In \emph{2021 IEEE International Conference on Robotics and Automation (ICRA)}, pp.\  1839--1845, 2021.
\newblock \doi{10.1109/ICRA48506.2021.9561461}.

\bibitem[Chen et~al.(2020)Chen, Wei, Huang, Ding, and Li]{gcnii}
Chen, M., Wei, Z., Huang, Z., Ding, B., and Li, Y.
\newblock Simple and deep graph convolutional networks.
\newblock In \emph{Proceedings of the 37th International Conference on Machine Learning}, volume 119 of \emph{Proceedings of Machine Learning Research}, pp.\  1725--1735. PMLR, 07 2020.
\newblock URL \url{https://proceedings.mlr.press/v119/chen20v.html}.

\bibitem[Chien et~al.(2021)Chien, Peng, Li, and Milenkovic]{GPRGNN}
Chien, E., Peng, J., Li, P., and Milenkovic, O.
\newblock Adaptive universal generalized pagerank graph neural network.
\newblock In \emph{International Conference on Learning Representations}, 2021.
\newblock URL \url{https://openreview.net/forum?id=n6jl7fLxrP}.

\bibitem[Chihara(2011)]{orthogonalpoly_2}
Chihara, T.~S.
\newblock \emph{An introduction to orthogonal polynomials}.
\newblock Courier Corporation, 2011.

\bibitem[Chung(1997)]{spectralgraphtheory}
Chung, F.
\newblock \emph{{Spectral Graph Theory}}, volume~92.
\newblock CBMS Regional Conference Series in Mathematics, 1997.
\newblock ISBN 978-0-8218-0315-8.
\newblock \doi{/10.1090/cbms/092}.

\bibitem[Ciregan et~al.(2012)Ciregan, Meier, and Schmidhuber]{fully-connected-2}
Ciregan, D., Meier, U., and Schmidhuber, J.
\newblock Multi-column deep neural networks for image classification.
\newblock In \emph{2012 IEEE Conference on Computer Vision and Pattern Recognition}, pp.\  3642--3649, 2012.
\newblock \doi{10.1109/CVPR.2012.6248110}.

\bibitem[Defferrard et~al.(2016{\natexlab{a}})Defferrard, Bresson, and Vandergheynst]{ChebNet}
Defferrard, M., Bresson, X., and Vandergheynst, P.
\newblock Convolutional neural networks on graphs with fast localized spectral filtering.
\newblock In Lee, D., Sugiyama, M., Luxburg, U., Guyon, I., and Garnett, R. (eds.), \emph{Advances in Neural Information Processing Systems}, volume~29. Curran Associates, Inc., 2016{\natexlab{a}}.

\bibitem[Defferrard et~al.(2016{\natexlab{b}})Defferrard, Bresson, and Vandergheynst]{spectralGNN-1}
Defferrard, M., Bresson, X., and Vandergheynst, P.
\newblock Convolutional neural networks on graphs with fast localized spectral filtering.
\newblock In Lee, D., Sugiyama, M., Luxburg, U., Guyon, I., and Garnett, R. (eds.), \emph{Advances in Neural Information Processing Systems}, volume~29. Curran Associates, Inc., 2016{\natexlab{b}}.
\newblock URL \url{https://proceedings.neurips.cc/paper_files/paper/2016/file/04df4d434d481c5bb723be1b6df1ee65-Paper.pdf}.

\bibitem[Dong et~al.(2020)Dong, Thanou, Toni, Bronstein, and Frossard]{GraphSignalProcessingforMachineLearningAReviewandNewPerspectives}
Dong, X., Thanou, D., Toni, L., Bronstein, M., and Frossard, P.
\newblock Graph signal processing for machine learning: A review and new perspectives.
\newblock \emph{IEEE Signal Processing Magazine}, 37\penalty0 (6):\penalty0 117--127, 2020.
\newblock \doi{10.1109/MSP.2020.3014591}.

\bibitem[Dong et~al.(2021)Dong, Ding, Jalaian, Ji, and Li]{Adagnn--spectral-2dgnn}
Dong, Y., Ding, K., Jalaian, B., Ji, S., and Li, J.
\newblock Adagnn: Graph neural networks with adaptive frequency response filter.
\newblock In \emph{Proceedings of the 30th ACM International Conference on Information \& Knowledge Management}, pp.\  392–401, New York, NY, USA, 2021. Association for Computing Machinery.
\newblock ISBN 9781450384469.
\newblock \doi{10.1145/3459637.3482226}.
\newblock URL \url{https://doi.org/10.1145/3459637.3482226}.

\bibitem[Duan(2015)]{generalized-sylvester-equation-3}
Duan, G.-R.
\newblock \emph{Generalized Sylvester equations: unified parametric solutions}.
\newblock Crc Press, 2015.

\bibitem[Eliasof et~al.(2021)Eliasof, Haber, and Treister]{pdegcn}
Eliasof, M., Haber, E., and Treister, E.
\newblock {PDE}-{GCN}: Novel architectures for graph neural networks motivated by partial differential equations.
\newblock In Beygelzimer, A., Dauphin, Y., Liang, P., and Vaughan, J.~W. (eds.), \emph{Advances in Neural Information Processing Systems}, 2021.
\newblock URL \url{https://openreview.net/forum?id=wWtk6GxJB2x}.

\bibitem[Epperson(1987)]{RungePhenomenon}
Epperson, J.~F.
\newblock On the runge example.
\newblock \emph{The American Mathematical Monthly}, 94\penalty0 (4):\penalty0 329--341, 1987.
\newblock \doi{10.1080/00029890.1987.12000642}.
\newblock URL \url{https://doi.org/10.1080/00029890.1987.12000642}.

\bibitem[Fey \& Lenssen(2019)Fey and Lenssen]{PyTorchGeometric}
Fey, M. and Lenssen, J.~E.
\newblock Fast graph representation learning with pytorch geometric, 2019.
\newblock URL \url{https://arxiv.org/abs/1903.02428}.

\bibitem[Flanders \& Wimmer(1977)Flanders and Wimmer]{generalized-sylvester-equation-1}
Flanders, H. and Wimmer, H.~K.
\newblock On the matrix equations ax-xb=c and ax-yb=c.
\newblock \emph{SIAM Journal on Applied Mathematics}, 32\penalty0 (4):\penalty0 707--710, 1977.

\bibitem[Gama et~al.(2020)Gama, Isufi, Leus, and Ribeiro]{graphconv_2}
Gama, F., Isufi, E., Leus, G., and Ribeiro, A.
\newblock Graphs, convolutions, and neural networks: From graph filters to graph neural networks.
\newblock \emph{IEEE Signal Processing Magazine}, 37\penalty0 (6):\penalty0 128--138, 2020.
\newblock \doi{10.1109/MSP.2020.3016143}.

\bibitem[Gasteiger et~al.(2019)Gasteiger, Bojchevski, and Günnemann]{APPNP}
Gasteiger, J., Bojchevski, A., and Günnemann, S.
\newblock Predict then propagate: Graph neural networks meet personalized pagerank.
\newblock In \emph{International Conference on Learning Representations}, 2019.
\newblock URL \url{https://openreview.net/forum?id=H1gL-2A9Ym}.

\bibitem[Gil et~al.(2007)Gil, Segura, and Temme]{chebyshevinterpolation}
Gil, A., Segura, J., and Temme, N.~M.
\newblock \emph{Numerical Methods for Special Functions}.
\newblock Society for Industrial and Applied Mathematics, 2007.
\newblock \doi{10.1137/1.9780898717822}.
\newblock URL \url{https://epubs.siam.org/doi/abs/10.1137/1.9780898717822}.

\bibitem[Gilmer et~al.(2017)Gilmer, Schoenholz, Riley, Vinyals, and Dahl]{MessagePNN}
Gilmer, J., Schoenholz, S.~S., Riley, P.~F., Vinyals, O., and Dahl, G.~E.
\newblock Neural message passing for quantum chemistry, 2017.
\newblock URL \url{https://arxiv.org/abs/1704.01212}.

\bibitem[Goodfellow et~al.(2016)Goodfellow, Bengio, and Courville]{deep-learning}
Goodfellow, I., Bengio, Y., and Courville, A.
\newblock \emph{Deep learning}.
\newblock MIT press, 2016.

\bibitem[Guo \& Wei(2023)Guo and Wei]{OptBasisGNN}
Guo, Y. and Wei, Z.
\newblock Graph neural networks with learnable and optimal polynomial bases.
\newblock In Krause, A., Brunskill, E., Cho, K., Engelhardt, B., Sabato, S., and Scarlett, J. (eds.), \emph{Proceedings of the 40th International Conference on Machine Learning}, volume 202 of \emph{Proceedings of Machine Learning Research}, pp.\  12077--12097. PMLR, 23--29 Jul 2023.
\newblock URL \url{https://proceedings.mlr.press/v202/guo23i.html}.

\bibitem[Hamilton et~al.(2018)Hamilton, Ying, and Leskovec]{representationlearningongraph}
Hamilton, W.~L., Ying, R., and Leskovec, J.
\newblock Representation learning on graphs: Methods and applications, 2018.

\bibitem[He et~al.(2021)He, Wei, Huang, and Xu]{BernNet-GNN-narrowbandresults-1}
He, M., Wei, Z., Huang, Z., and Xu, H.
\newblock Bernnet: Learning arbitrary graph spectral filters via bernstein approximation.
\newblock In Beygelzimer, A., Dauphin, Y., Liang, P., and Vaughan, J.~W. (eds.), \emph{Advances in Neural Information Processing Systems}, 2021.
\newblock URL \url{https://openreview.net/forum?id=WigDnV-_Gq}.

\bibitem[He et~al.(2022{\natexlab{a}})He, Wei, and Wen]{ChebNetII}
He, M., Wei, Z., and Wen, J.-R.
\newblock Convolutional neural networks on graphs with chebyshev approximation, revisited.
\newblock In Oh, A.~H., Agarwal, A., Belgrave, D., and Cho, K. (eds.), \emph{Advances in Neural Information Processing Systems}, 2022{\natexlab{a}}.
\newblock URL \url{https://openreview.net/forum?id=jxPJ4QA0KAb}.

\bibitem[He et~al.(2022{\natexlab{b}})He, Perlmutter, Reinert, and Cucuringu]{spectralGNN-5}
He, Y., Perlmutter, M., Reinert, G., and Cucuringu, M.
\newblock Msgnn: A spectral graph neural network based on a novel magnetic signed laplacian.
\newblock In Rieck, B. and Pascanu, R. (eds.), \emph{Proceedings of the First Learning on Graphs Conference}, volume 198 of \emph{Proceedings of Machine Learning Research}, pp.\  40:1--40:39. PMLR, 09--12 Dec 2022{\natexlab{b}}.
\newblock URL \url{https://proceedings.mlr.press/v198/he22c.html}.

\bibitem[Hu et~al.(2020)Hu, Fey, Zitnik, Dong, Ren, Liu, Catasta, and Leskovec]{dataset5-ogb}
Hu, W., Fey, M., Zitnik, M., Dong, Y., Ren, H., Liu, B., Catasta, M., and Leskovec, J.
\newblock Open graph benchmark: Datasets for machine learning on graphs.
\newblock In \emph{Advances in Neural Information Processing Systems}, volume~33, pp.\  22118--22133, 2020.
\newblock URL \url{https://proceedings.neurips.cc/paper_files/paper/2020/file/fb60d411a5c5b72b2e7d3527cfc84fd0-Paper.pdf}.

\bibitem[Huang et~al.(2023)Huang, Du, Chen, Fu, Han, and Zhang]{robusempGCN}
Huang, J., Du, L., Chen, X., Fu, Q., Han, S., and Zhang, D.
\newblock Robust mid-pass filtering graph convolutional networks.
\newblock In \emph{Proceedings of the ACM Web Conference 2023}, WWW '23, pp.\  328–338. Association for Computing Machinery, 2023.
\newblock ISBN 9781450394161.
\newblock \doi{10.1145/3543507.3583335}.
\newblock URL \url{https://doi.org/10.1145/3543507.3583335}.

\bibitem[Kingma \& Ba(2014)Kingma and Ba]{Adamoptimizer}
Kingma, D.~P. and Ba, J.
\newblock Adam: A method for stochastic optimization, 2014.
\newblock URL \url{https://arxiv.org/abs/1412.6980}.

\bibitem[Kipf \& Welling(2017)Kipf and Welling]{GCN}
Kipf, T.~N. and Welling, M.
\newblock Semi-supervised classification with graph convolutional networks.
\newblock In \emph{International Conference on Learning Representations}, 2017.
\newblock URL \url{https://openreview.net/forum?id=SJU4ayYgl}.

\bibitem[Krizhevsky et~al.(2012)Krizhevsky, Sutskever, and Hinton]{fully-connected-3}
Krizhevsky, A., Sutskever, I., and Hinton, G.~E.
\newblock Imagenet classification with deep convolutional neural networks.
\newblock In Pereira, F., Burges, C., Bottou, L., and Weinberger, K. (eds.), \emph{Advances in Neural Information Processing Systems}, volume~25. Curran Associates, Inc., 2012.
\newblock URL \url{https://proceedings.neurips.cc/paper_files/paper/2012/file/c399862d3b9d6b76c8436e924a68c45b-Paper.pdf}.

\bibitem[Kurokawa et~al.(2017)Kurokawa, Oki, and Nagao]{2dGFT}
Kurokawa, T., Oki, T., and Nagao, H.
\newblock Multi-dimensional graph fourier transform, 2017.

\bibitem[Lei et~al.(2022)Lei, Wang, Li, Ding, and Wei]{EvenNet}
Lei, R., Wang, Z., Li, Y., Ding, B., and Wei, Z.
\newblock Evennet: Ignoring odd-hop neighbors improves robustness of graph neural networks.
\newblock In Koyejo, S., Mohamed, S., Agarwal, A., Belgrave, D., Cho, K., and Oh, A. (eds.), \emph{Advances in Neural Information Processing Systems}, volume~35, pp.\  4694--4706. Curran Associates, Inc., 2022.
\newblock URL \url{https://openreview.net/forum?id=SPoiDLr3WE7}.

\bibitem[Li et~al.(2021{\natexlab{a}})Li, Wang, and Zhang]{fully-connected-5}
Li, P., Wang, B., and Zhang, L.
\newblock Virtual fully-connected layer: Training a large-scale face recognition dataset with limited computational resources.
\newblock In \emph{Proceedings of the IEEE/CVF Conference on Computer Vision and Pattern Recognition (CVPR)}, pp.\  13315--13324, June 2021{\natexlab{a}}.

\bibitem[Li et~al.(2021{\natexlab{b}})Li, Zhang, Liu, Dai, and Wu]{DSGC-spectral-2dgnn-2dGFT-gnn-2}
Li, Q., Zhang, X., Liu, H., Dai, Q., and Wu, X.-M.
\newblock Dimensionwise separable 2-d graph convolution for unsupervised and semi-supervised learning on graphs.
\newblock In \emph{Proceedings of the 27th ACM SIGKDD Conference on Knowledge Discovery \& Data Mining}, pp.\  953–963. Association for Computing Machinery, 2021{\natexlab{b}}.
\newblock ISBN 9781450383325.
\newblock \doi{10.1145/3447548.3467413}.
\newblock URL \url{https://doi.org/10.1145/3447548.3467413}.

\bibitem[Li et~al.(2022)Li, Zhu, Cheng, Shan, Luo, Li, and Qian]{glognn++}
Li, X., Zhu, R., Cheng, Y., Shan, C., Luo, S., Li, D., and Qian, W.
\newblock Finding global homophily in graph neural networks when meeting heterophily.
\newblock In Chaudhuri, K., Jegelka, S., Song, L., Szepesvari, C., Niu, G., and Sabato, S. (eds.), \emph{Proceedings of the 39th International Conference on Machine Learning}, volume 162 of \emph{Proceedings of Machine Learning Research}, pp.\  13242--13256. PMLR, 17--23 Jul 2022.
\newblock URL \url{https://proceedings.mlr.press/v162/li22ad.html}.

\bibitem[Lim et~al.(2021)Lim, Hohne, Li, Huang, Gupta, Bhalerao, and Lim]{dataset6-large-hetero}
Lim, D., Hohne, F.~M., Li, X., Huang, S.~L., Gupta, V., Bhalerao, O.~P., and Lim, S.-N.
\newblock Large scale learning on non-homophilous graphs: New benchmarks and strong simple methods.
\newblock In Beygelzimer, A., Dauphin, Y., Liang, P., and Vaughan, J.~W. (eds.), \emph{Advances in Neural Information Processing Systems}, 2021.
\newblock URL \url{https://openreview.net/forum?id=DfGu8WwT0d}.

\bibitem[Lin et~al.(2015)Lin, Memisevic, and Konda]{fully-connected-4}
Lin, Z., Memisevic, R., and Konda, K.
\newblock How far can we go without convolution: Improving fully-connected networks, 2015.

\bibitem[Ma \& Lu(2017)Ma and Lu]{FC=Conv-1}
Ma, W. and Lu, J.
\newblock An equivalence of fully connected layer and convolutional layer, 2017.

\bibitem[Monti et~al.(2017)Monti, Bronstein, and Bresson]{2dGFT-gnn-1}
Monti, F., Bronstein, M.~M., and Bresson, X.
\newblock Geometric matrix completion with recurrent multi-graph neural networks.
\newblock In \emph{Proceedings of the 31st International Conference on Neural Information Processing Systems}, pp.\  3700–3710, Red Hook, NY, USA, 2017. Curran Associates Inc.
\newblock ISBN 9781510860964.

\bibitem[Ortega et~al.(2018)Ortega, Frossard, Kovačević, Moura, and Vandergheynst]{GraphSignalProcessingOverviewChallengesandApplications}
Ortega, A., Frossard, P., Kovačević, J., Moura, J. M.~F., and Vandergheynst, P.
\newblock Graph signal processing: Overview, challenges, and applications.
\newblock \emph{Proceedings of the IEEE}, 106\penalty0 (5):\penalty0 808--828, 2018.
\newblock \doi{10.1109/JPROC.2018.2820126}.

\bibitem[Page et~al.(1999)Page, Brin, Motwani, and Winograd]{pagerank}
Page, L., Brin, S., Motwani, R., and Winograd, T.
\newblock The pagerank citation ranking : Bringing order to the web.
\newblock \emph{Technical report, stanford University}, 1999.

\bibitem[Pei et~al.(2020)Pei, Wei, Chang, Lei, and Yang]{dataset3-pei}
Pei, H., Wei, B., Chang, K. C.-C., Lei, Y., and Yang, B.
\newblock Geom-gcn: Geometric graph convolutional networks.
\newblock In \emph{International Conference on Learning Representations}, 2020.
\newblock URL \url{https://openreview.net/forum?id=S1e2agrFvS}.

\bibitem[Petersen \& Voigtlaender(2020)Petersen and Voigtlaender]{FC=Conv-2}
Petersen, P. and Voigtlaender, F.
\newblock Equivalence of approximation by convolutional neural networks and fully-connected networks.
\newblock \emph{Proceedings of the American Mathematical Society}, 148\penalty0 (4):\penalty0 1567--1581, 2020.

\bibitem[Phillips(2003)]{polyapprox_2}
Phillips, G.~M.
\newblock \emph{Interpolation and approximation by polynomials}, volume~14.
\newblock Springer New York, 2003.
\newblock ISBN 978-0-387-00215-6.
\newblock \doi{10.1007/b97417}.

\bibitem[Platonov et~al.(2023)Platonov, Kuznedelev, Diskin, Babenko, and Prokhorenkova]{dataset8-small-hetero}
Platonov, O., Kuznedelev, D., Diskin, M., Babenko, A., and Prokhorenkova, L.
\newblock A critical look at the evaluation of {GNN}s under heterophily: Are we really making progress?
\newblock In \emph{The Eleventh International Conference on Learning Representations}, 2023.
\newblock URL \url{https://openreview.net/forum?id=tJbbQfw-5wv}.

\bibitem[Rivlin(2020)]{poly_chebyshev}
Rivlin, T.~J.
\newblock \emph{Chebyshev polynomials}.
\newblock Courier Dover Publications, 2020.

\bibitem[Roth et~al.(1952)]{William-E-Roth-condition-sylvester-equation}
Roth, W.~E. et~al.
\newblock The equations ax- yb= c and ax- xb= c in matrices.
\newblock In \emph{Proc. Amer. Math. Soc}, volume~3, pp.\  392--396, 1952.

\bibitem[Rozemberczki et~al.(2021)Rozemberczki, Allen, and Sarkar]{dataset4-cham-squi}
Rozemberczki, B., Allen, C., and Sarkar, R.
\newblock Multi-scale attributed node embedding.
\newblock \emph{Journal of Complex Networks}, 9\penalty0 (2), 05 2021.
\newblock \doi{10.1093/comnet/cnab014}.
\newblock URL \url{https://doi.org/10.1093/comnet/cnab014}.

\bibitem[Sermanet \& LeCun(2011)Sermanet and LeCun]{fully-connected-1}
Sermanet, P. and LeCun, Y.
\newblock Traffic sign recognition with multi-scale convolutional networks.
\newblock In \emph{The 2011 International Joint Conference on Neural Networks}, pp.\  2809--2813, 2011.
\newblock \doi{10.1109/IJCNN.2011.6033589}.

\bibitem[Shchur et~al.(2019)Shchur, Mumme, Bojchevski, and Günnemann]{dataset2-photo-comp}
Shchur, O., Mumme, M., Bojchevski, A., and Günnemann, S.
\newblock Pitfalls of graph neural network evaluation, 2019.

\bibitem[Shi \& Moura(2019)Shi and Moura]{graphconv_1}
Shi, J. and Moura, J.~M.
\newblock Topics in graph signal processing: Convolution and modulation.
\newblock In \emph{2019 53rd Asilomar Conference on Signals, Systems, and Computers}, pp.\  457--461, 2019.
\newblock \doi{10.1109/IEEECONF44664.2019.9049012}.

\bibitem[Smyth(1998)]{polyapprox_1}
Smyth, G.~K.
\newblock Polynomial approximation.
\newblock \emph{Encyclopedia of Biostatistics}, 13, 1998.

\bibitem[Strang(2006)]{linearalgebra}
Strang, G.
\newblock \emph{Linear algebra and its applications.}
\newblock Belmont, CA: Thomson, Brooks/Cole, 2006.

\bibitem[Szeg(1939)]{orthogonalpoly_1_poly_jacobian}
Szeg, G.
\newblock \emph{Orthogonal polynomials}, volume~23.
\newblock American Mathematical Soc., 1939.

\bibitem[Tang et~al.(2009)Tang, Sun, Wang, and Yang]{dataset7-actor}
Tang, J., Sun, J., Wang, C., and Yang, Z.
\newblock Social influence analysis in large-scale networks.
\newblock In \emph{Proceedings of the 15th ACM SIGKDD International Conference on Knowledge Discovery and Data Mining}, KDD '09, pp.\  807–816, New York, NY, USA, 2009. Association for Computing Machinery.
\newblock ISBN 9781605584959.
\newblock \doi{10.1145/1557019.1557108}.
\newblock URL \url{https://doi.org/10.1145/1557019.1557108}.

\bibitem[Tang et~al.(2022)Tang, Li, Gao, and Li]{spectralGNN-6}
Tang, J., Li, J., Gao, Z., and Li, J.
\newblock Rethinking graph neural networks for anomaly detection.
\newblock In \emph{Proceedings of the 39th International Conference on Machine Learning}, volume 162 of \emph{Proceedings of Machine Learning Research}, pp.\  21076--21089. PMLR, 07 2022.
\newblock URL \url{https://proceedings.mlr.press/v162/tang22b.html}.

\bibitem[Wang et~al.(2018)Wang, Samari, and Siddiqi]{spectralGNN-2}
Wang, C., Samari, B., and Siddiqi, K.
\newblock Local spectral graph convolution for point set feature learning.
\newblock In \emph{Proceedings of the European Conference on Computer Vision (ECCV)}, September 2018.

\bibitem[Wang \& Zhang(2022)Wang and Zhang]{JacobiConv}
Wang, X. and Zhang, M.
\newblock How powerful are spectral graph neural networks.
\newblock In Chaudhuri, K., Jegelka, S., Song, L., Szepesvari, C., Niu, G., and Sabato, S. (eds.), \emph{Proceedings of the 39th International Conference on Machine Learning}, volume 162 of \emph{Proceedings of Machine Learning Research}, pp.\  23341--23362. PMLR, Jul 2022.
\newblock URL \url{https://proceedings.mlr.press/v162/wang22am.html}.

\bibitem[Wu et~al.(2019)Wu, Souza, Zhang, Fifty, Yu, and Weinberger]{SGC}
Wu, F., Souza, A., Zhang, T., Fifty, C., Yu, T., and Weinberger, K.
\newblock Simplifying graph convolutional networks.
\newblock In \emph{Proceedings of the 36th International Conference on Machine Learning}, volume~97, pp.\  6861--6871. PMLR, 06 2019.
\newblock URL \url{https://proceedings.mlr.press/v97/wu19e.html}.

\bibitem[Wu et~al.(2022)Wu, Zhao, Li, Wipf, and Yan]{nodeformer}
Wu, Q., Zhao, W., Li, Z., Wipf, D., and Yan, J.
\newblock Nodeformer: A scalable graph structure learning transformer for node classification.
\newblock In Oh, A.~H., Agarwal, A., Belgrave, D., and Cho, K. (eds.), \emph{Advances in Neural Information Processing Systems}, 2022.
\newblock URL \url{https://openreview.net/forum?id=sMezXGG5So}.

\bibitem[Wu et~al.(2021)Wu, Pan, Chen, Long, Zhang, and Yu]{comprehensivegnn}
Wu, Z., Pan, S., Chen, F., Long, G., Zhang, C., and Yu, P.~S.
\newblock A comprehensive survey on graph neural networks.
\newblock \emph{IEEE Transactions on Neural Networks and Learning Systems}, 32\penalty0 (1):\penalty0 4--24, 2021.
\newblock \doi{10.1109/TNNLS.2020.2978386}.

\bibitem[Xu et~al.(2019)Xu, Shen, Cao, Qiu, and Cheng]{GWNN}
Xu, B., Shen, H., Cao, Q., Qiu, Y., and Cheng, X.
\newblock Graph wavelet neural network.
\newblock In \emph{International Conference on Learning Representations}, 2019.
\newblock URL \url{https://openreview.net/forum?id=H1ewdiR5tQ}.

\bibitem[Yang et~al.(2022)Yang, Shen, Li, Qi, Zhang, and Yin]{specGN}
Yang, M., Shen, Y., Li, R., Qi, H., Zhang, Q., and Yin, B.
\newblock A new perspective on the effects of spectrum in graph neural networks.
\newblock In \emph{Proceedings of the 39th International Conference on Machine Learning}, volume 162 of \emph{Proceedings of Machine Learning Research}, pp.\  25261--25279. PMLR, 17--23 Jul 2022.
\newblock URL \url{https://proceedings.mlr.press/v162/yang22n.html}.

\bibitem[Yang et~al.(2016)Yang, Cohen, and Salakhutdinov]{dataset1-cora}
Yang, Z., Cohen, W.~W., and Salakhutdinov, R.
\newblock Revisiting semi-supervised learning with graph embeddings, 2016.
\newblock URL \url{https://arxiv.org/abs/1603.08861}.

\bibitem[Zhang et~al.(2022)Zhang, Cui, and Zhu]{deepgraphlearningsurvey}
Zhang, Z., Cui, P., and Zhu, W.
\newblock Deep learning on graphs: A survey.
\newblock \emph{IEEE Transactions on Knowledge and Data Engineering}, 34\penalty0 (1):\penalty0 249--270, 2022.
\newblock \doi{10.1109/TKDE.2020.2981333}.

\bibitem[Zhao et~al.(2021)Zhao, Dong, Ding, Kharlamov, and Tang]{ADC}
Zhao, J., Dong, Y., Ding, M., Kharlamov, E., and Tang, J.
\newblock Adaptive diffusion in graph neural networks.
\newblock In Beygelzimer, A., Dauphin, Y., Liang, P., and Vaughan, J.~W. (eds.), \emph{Advances in Neural Information Processing Systems}, 2021.
\newblock URL \url{https://openreview.net/forum?id=0Kb33DHJ1g}.

\bibitem[Zheng et~al.(2023)Zheng, Zhu, Liu, Li, and Zhao]{NFGNN}
Zheng, S., Zhu, Z., Liu, Z., Li, Y., and Zhao, Y.
\newblock Node-oriented spectral filtering for graph neural networks.
\newblock \emph{IEEE Transactions on Pattern Analysis and Machine Intelligence}, 46\penalty0 (1):\penalty0 388--402, 2023.
\newblock \doi{10.1109/TPAMI.2023.3324937}.

\bibitem[Zhu \& Koniusz(2021)Zhu and Koniusz]{SSGC}
Zhu, H. and Koniusz, P.
\newblock Simple spectral graph convolution.
\newblock In \emph{International Conference on Learning Representations}, 2021.
\newblock URL \url{https://openreview.net/forum?id=CYO5T-YjWZV}.

\bibitem[Zhu et~al.(2021)Zhu, Wang, Shi, Ji, and Cui]{GNN-HF-LF}
Zhu, M., Wang, X., Shi, C., Ji, H., and Cui, P.
\newblock Interpreting and unifying graph neural networks with an optimization framework.
\newblock In \emph{Proceedings of the Web Conference 2021}, WWW '21, pp.\  1215–1226. Association for Computing Machinery, 2021.
\newblock ISBN 9781450383127.
\newblock \doi{10.1145/3442381.3449953}.
\newblock URL \url{https://doi.org/10.1145/3442381.3449953}.

\bibitem[Zhu et~al.(2022)Zhu, Xu, Zhang, Du, Zhang, Liu, Yang, and Wu]{surveygraphstructurelearning}
Zhu, Y., Xu, W., Zhang, J., Du, Y., Zhang, J., Liu, Q., Yang, C., and Wu, S.
\newblock A survey on graph structure learning: Progress and opportunities, 2022.

\end{thebibliography}
\bibliographystyle{icml2023}

\newpage
\appendix
\onecolumn


\section{Proofs and Derivations}
\label{appendix-proofs-and-derivations}

\subsection{Proof of Theorem~\ref{theorem-existing-paradigms-fails-to-construct-optimal-output}}
\label{appendix-proofs-and-derivations-theorem-existing-paradigms-fails-to-construct-optimal-output}
Let $\Hat{\boldsymbol{Z}}^{*}=\boldsymbol{U}^{T}\boldsymbol{Z}^{*}$, $\Hat{\boldsymbol{F}}=\boldsymbol{U}^{T}\boldsymbol{F}$ denote the GFT of $\boldsymbol{Z}^{*}$, $\boldsymbol{F}$, respectively. 
Note that the statement is equivalent to finding the necessary conditions of making $0$ construction error possible. 
Thus, we alternatively investigate the equivalent statement in the order of Paradigm {\bf (I)}, {\bf (II)} and {\bf (III)}. 
\par {\bf \ding{172}-On Paradigm (I) type graph convolution:} 
\par To achieve $0$ construction error with the Paradigm {\bf (I)} type graph convolution, the following equations must hold: 
\begin{align}
\label{equation-graphconv-1-fail-1}
&(\text{Matrix scale})\ \boldsymbol{Z}^{*}=\boldsymbol{\Phi}_{\mathcal{G}}\boldsymbol{F}\ , \\
\label{equation-graphconv-1-fail-2}
&(\text{Element scale})\ \Hat{\boldsymbol{Z}}^{*}_{nc}=\boldsymbol{g}_{n}\Hat{\boldsymbol{F}}_{nc}\ , 
\end{align}
where $\boldsymbol{\Phi}_{\mathcal{G}}=\boldsymbol{U}diag(\boldsymbol{g})\boldsymbol{U}^{T}$, $\boldsymbol{g}\in\mathbb{R}^{N}$. 
The equations above indicate the necessary conditions for possible $0$ construction error as follows:
\begin{align}
\label{equation-graphconv-1-fail-3}
&\text{Eq.~\ref{equation-graphconv-1-fail-1}}\Longrightarrow rank(\boldsymbol{Z}^{*})=rank(\boldsymbol{\Phi}_{\mathcal{G}}\boldsymbol{F})\leq rank(\boldsymbol{F})\ , \\
\label{equation-graphconv-1-fail-4}
&\text{Eq.~\ref{equation-graphconv-1-fail-2}}\Longrightarrow \Pr\left(\Hat{\boldsymbol{Z}}^{*}_{nc}\neq0|\Hat{\boldsymbol{F}}_{nc}=0\right)=0\ . 
\end{align}
That is, both Eq.~\ref{equation-graphconv-1-fail-3} and Eq.~\ref{equation-graphconv-1-fail-4} are necessary conditions for achieving possible $0$ construction error with the Paradigm {\bf I} type graph convolution. 
Accordingly, the graph convolution as Paradigm {\bf (I)} fails to achieve $0$ error when 
\begin{equation}
\label{equation-graphconv-1-fail-5}
\left\{rank(\boldsymbol{Z}^{*})>rank(\boldsymbol{F})\right\}\bigcup\left\{\Pr\left(\Hat{\boldsymbol{Z}}^{*}_{nc}\neq0|\Hat{\boldsymbol{F}}_{nc}=0\right)>0\right\}\ , 
\end{equation}
where statement holds for all parameter values of $\boldsymbol{\Phi}_{\mathcal{G}}$. 

\par {\bf \ding{173}-On Paradigm (II) type graph convolution:} 
\par To achieve $0$ construction error with the Paradigm {\bf (II)} type graph convolution, the following equations must hold:
\begin{align}
\label{equation-graphconv-2-fail-1}
\boldsymbol{Z}^{*}=\boldsymbol{\Phi}_{\mathcal{G}}\boldsymbol{F}\boldsymbol{R}\ . 
\end{align}
Similar to Eq.~\ref{equation-graphconv-1-fail-1}, the above equation indicates the necessary conditions as follows: 
\begin{align}
\label{equation-graphconv-2-fail-2}
\text{Eq.~\ref{equation-graphconv-2-fail-1}}\Longrightarrow rank(\boldsymbol{Z}^{*})=rank(\boldsymbol{\Phi}_{\mathcal{G}}\boldsymbol{F}\boldsymbol{R})\leq \min\{rank(\boldsymbol{F}), rank(\boldsymbol{R})\}\leq rank(\boldsymbol{F})\ .
\end{align}
Thus, the Paradigm {\bf (II)} type graph convolution fails to achieve $0$ error when $rank(\boldsymbol{Z}^{*})>rank(\boldsymbol{F})$, where statement holds for all parameter values of $\left\{\boldsymbol{\Phi}_{\mathcal{G}}, \boldsymbol{R}\right\}$. 

\par {\bf \ding{174}-On Paradigm (III) type graph convolution:} 
\par To achieve $0$ construction error with the Paradigm {\bf (III)} type graph convolution, the following equation must hold: 
\begin{align}
\label{equation-graphconv-3-fail-1}
\boldsymbol{Z}^{*}_{:c}=\boldsymbol{\Phi}^{(c)}_{\mathcal{G}}\boldsymbol{F}_{:c}\ \Longleftrightarrow\ \Hat{\boldsymbol{Z}}^{*}_{nc}=\boldsymbol{g}^{(c)}_{n}\Hat{\boldsymbol{F}}_{nc}\ ,\quad c=1,2,...,C\ .
\end{align}
Note that $\boldsymbol{g}^{(c)}_{n}\in\mathbb{R}$ is unique towards $\{\Hat{\boldsymbol{F}}_{nc},\Hat{\boldsymbol{Z}}^{*}_{nc}\}$, demonstrating the condition as follows: 
\begin{align}
\label{equation-graphconv-3-fail-2}
\Pr\left(\Hat{\boldsymbol{Z}}^{*}_{nc}\neq0|\Hat{\boldsymbol{F}}_{nc}=0\right)=0\ . 
\end{align}
Thus, the Paradigm {\bf (III)} type graph convolution fails to achieve $0$ error when $\Pr\left(\Hat{\boldsymbol{Z}}^{*}_{nc}\neq0|\Hat{\boldsymbol{F}}_{nc}=0\right)>0$, where statement holds for all parameter values of $\left\{\boldsymbol{\Phi}^{(1)}_{\mathcal{G}},\boldsymbol{\Phi}^{(2)}_{\mathcal{G}},...,\boldsymbol{\Phi}^{(C)}_{\mathcal{G}}\right\}$. 

\par Thus, with \ding{172}, \ding{173} and \ding{174} above, we demonstrate the failure cases of the three paradigms, proving that those paradigms \textbf{never} construct the target output with certain conditions.

\subsection{Proof of Corollary~\ref{corollary-simple-sum-still-failures}}
\label{appendix-proofs-and-derivations-corollary-simple-sum-still-failures}
\begin{proof}
\label{appendix-proofs-and-derivations-proof:corollary-simple-sum-still-failures}
As the graph convolution operator $\boldsymbol{\Phi}_{\mathcal{G}}$ can be decomposed as $\boldsymbol{\Phi}_{\mathcal{G}}=\boldsymbol{U}diag(\boldsymbol{g})\boldsymbol{U}^{T}$, we can thereby investigate an equivalent statement on frequency domain, i.e., constructing target $\Hat{\boldsymbol{Z}}^{*}=\boldsymbol{U}^{T}\boldsymbol{Z}^{*}$ with graph signal matrix $\Hat{\boldsymbol{F}}=\boldsymbol{U}^{T}\boldsymbol{F}$, where the three paradigms are reformulated as follows: 
\begin{align}
\label{equation-proof:colloary-paradigm-1}
&\text{Paradigm {\bf (I)}:}\ \Hat{\boldsymbol{Z}}_{:c}=diag(\boldsymbol{g})\Hat{\boldsymbol{F}}_{:c}\ ,\quad c=1,2,...,C\ , \\
\label{equation-proof:colloary-paradigm-2}
&\text{Paradigm {\bf (II)}:}\ \Hat{\boldsymbol{Z}}_{:c}=\sum_{j=1}^{C}\boldsymbol{R}_{jc}diag(\boldsymbol{g})\Hat{\boldsymbol{F}}_{:j}\ ,\quad c=1,2,...,C\ , \\
\label{equation-proof:colloary-paradigm-3}
&\text{Paradigm {\bf (III)}:}\ \Hat{\boldsymbol{Z}}_{:c}=diag(\boldsymbol{g}^{(c)})\Hat{\boldsymbol{F}}_{:c}\ ,\quad c=1,2,...,C \ . 
\end{align}
On the other hand, we have two insights as follows: 
\begin{itemize}[leftmargin=0.5in]
\item Paradigm {\bf (I)} can be regarded as the special case of Paradigm {\bf (II)} or {\bf (III)}; 
\item The sum of any two different paradigms $(\mathcal{T}^{(i)}+\mathcal{T}^{(j)})$, $i,j\in\{1,2,3\}$, $i\neq j$, can be regarded as a special case of the sum of Paradigm {\bf (I)} and {\bf (II)} and {\bf (III)}, i.e., $(\mathcal{T}^{(1)}+\mathcal{T}^{(2)}+\mathcal{T}^{(3)})$. 
\end{itemize}
Thus, the proof of Corollary~\ref{corollary-simple-sum-still-failures} is equivalent to prove the statement holds for three different cases: {\bf \ding{172}-$(\mathcal{T}^{(2)}+\mathcal{T}^{(2)})$}; {\bf \ding{173}-$(\mathcal{T}^{(3)}+\mathcal{T}^{(3)})$}; {\bf \ding{174}-$(\mathcal{T}^{(1)}+\mathcal{T}^{(2)}+\mathcal{T}^{(3)})$}. 
\par {\bf \ding{172}-For $(\mathcal{T}^{(2)}+\mathcal{T}^{(2)})$:} 
\par Based on Eq.~\ref{equation-proof:colloary-paradigm-2}, the target output construction with $(\mathcal{T}^{(2)}+\mathcal{T}^{(2)})$ can be formulated as follows: 
\begin{align}
\label{equation-proof:colloary-case-I-1}
\Hat{\boldsymbol{Z}}^{*}=diag(\boldsymbol{p})\Hat{\boldsymbol{F}}\boldsymbol{W} + diag(\boldsymbol{q})\Hat{\boldsymbol{F}}\boldsymbol{V}\ .
\end{align}
Without loss of generality, we suppose both $diag(\boldsymbol{p})$ and $\boldsymbol{V}$ are invertible. 
Thus, Eq.~\ref{equation-proof:colloary-case-I-1} can be reformulated as follows:
\begin{align}
\label{equation-proof:colloary-case-I-2}
diag(\boldsymbol{p}^{-1})\Hat{\boldsymbol{Z}}^{*}\boldsymbol{V}^{-1}=\Hat{\boldsymbol{F}}\boldsymbol{W}\boldsymbol{V}^{-1} + diag(\boldsymbol{p}^{-1})diag(\boldsymbol{q})\Hat{\boldsymbol{F}}\ .
\end{align}
Note that solving Eq.~\ref{equation-proof:colloary-case-I-2} can be divided into two steps: \textbf{first:} determining all possible values of $\boldsymbol{p}^{-1}$, $\boldsymbol{V}^{-1}$; \textbf{second:} for all values of $\boldsymbol{p}^{-1}$ and $\boldsymbol{V}^{-1}$, finding out $diag(\boldsymbol{q})$ and $\boldsymbol{W}$ that make Eq.~\ref{equation-proof:colloary-case-I-2} hold. 
Accordingly, the $diag(\boldsymbol{p}^{-1})diag(\boldsymbol{q})$ and $\boldsymbol{W}\boldsymbol{V}^{-1}$ in the \textbf{second} step are equivalent to $diag(\boldsymbol{q})$ and $\boldsymbol{W}$ from the perspective of variables. 
Hence, we can replace $diag(\boldsymbol{p}^{-1})diag(\boldsymbol{q})$ and $\boldsymbol{W}\boldsymbol{V}^{-1}$ with $diag(\boldsymbol{q})$ and $\boldsymbol{W}$, and further achieve alternative formulation of Eq.~\ref{equation-proof:colloary-case-I-2} as follows:
\begin{align}
\label{equation-proof:colloary-case-I-3}
diag(\boldsymbol{p}^{-1})\Hat{\boldsymbol{Z}}^{*}\boldsymbol{V}^{-1}=\Hat{\boldsymbol{F}}\boldsymbol{W} + diag(\boldsymbol{q})\Hat{\boldsymbol{F}}\ . 
\end{align}
Thus, in the \textbf{second} step of solving Eq.~\ref{equation-proof:colloary-case-I-2}, we can regard Eq.~\ref{equation-proof:colloary-case-I-3} as a \textit{generalized Sylvester matrix equation}~\cite{generalized-sylvester-equation-1,generalized-sylvester-equation-2,generalized-sylvester-equation-3} which is formulated as follows:
\begin{align}
\label{equation-definition-sylvester-equation}
\boldsymbol{A}\boldsymbol{X}-\boldsymbol{Y}\boldsymbol{B}=\boldsymbol{C}\ ,
\end{align}
where $\boldsymbol{A}$, $\boldsymbol{B}$ and $\boldsymbol{C}$ are given matrices, both $\boldsymbol{X}$ and $\boldsymbol{Y}$ are the solutions of the matrix equation. 
The necessary and sufficient conditions for the solvability of Eq.~\ref{equation-definition-sylvester-equation} are given by William E. Roth, which can be restated as follows: 
\begin{theorem}
\label{theorem-roth-solution-to-sylvester-equation}
(\textbf{Solvability of generalized Sylvester equation~\cite{William-E-Roth-condition-sylvester-equation}.}) Solutions of Eq.~\ref{equation-definition-sylvester-equation} exist \textbf{if and only if} the matrices
\begin{align}
\label{equation-theorem-roth-solution-to-sylvester-equation}
\left[\begin{array}{cc}
   \boldsymbol{A}  & \boldsymbol{C} \\
    \boldsymbol{O} & \boldsymbol{B}
\end{array}\right]\quad\text{and}\quad\left[\begin{array}{cc}
   \boldsymbol{A}  & \boldsymbol{O} \\
    \boldsymbol{O} & \boldsymbol{B}
\end{array}\right]
\end{align}
are equivalent. 
\end{theorem}
Note that based on Eq.~\ref{equation-proof:colloary-case-I-3}, the sufficient and necessary condition of constructing $\Hat{\boldsymbol{Z}}^{*}$ with $(\mathcal{T}^{(2)}+\mathcal{T}^{(2)})$ is that there exists $\boldsymbol{p}^{-1}$, $\boldsymbol{V}^{-1}$, making the solutions of Eq.~\ref{equation-proof:colloary-case-I-3}, i.e., $diag(\boldsymbol{q})$ and $\boldsymbol{W}$, exist. 
Therefore, based on Theorem~\ref{theorem-roth-solution-to-sylvester-equation}, $\Hat{\boldsymbol{Z}}^{*}$ can be constructed with $(\mathcal{T}^{(2)}+\mathcal{T}^{(2)})$ \textbf{if and only if} there exists $\boldsymbol{p}^{-1}$, $\boldsymbol{V}^{-1}$ making the matrices
\begin{align}
\label{equation-proof:colloary-case-I-condition}
\left[\begin{array}{cc}
   \Hat{\boldsymbol{F}}  & diag(\boldsymbol{p}^{-1})\Hat{\boldsymbol{Z}}^{*}\boldsymbol{V}^{-1} \\
    \boldsymbol{O} & \Hat{\boldsymbol{F}}
\end{array}\right]\quad\text{and}\quad\left[\begin{array}{cc}
   \Hat{\boldsymbol{F}}  & \boldsymbol{O} \\
    \boldsymbol{O} & \Hat{\boldsymbol{F}}
\end{array}\right]
\end{align}
be equivalent. 
This indicates that the $(\mathcal{T}^{(2)}+\mathcal{T}^{(2)})$ fails to construct target output $\Hat{\boldsymbol{Z}}^{*}$ when the matrices in Eq.~\ref{equation-proof:colloary-case-I-condition} are \textbf{not} equal in rank, that is, 
\begin{align}
\label{equation-proof:colloary-case-I-fail-1}
rank(\left[\begin{array}{cc}
   \Hat{\boldsymbol{F}}  & diag(\boldsymbol{p}^{-1})\Hat{\boldsymbol{Z}}^{*}\boldsymbol{V}^{-1} \\
    \boldsymbol{O} & \Hat{\boldsymbol{F}}
\end{array}\right])>rank(\left[\begin{array}{cc}
   \Hat{\boldsymbol{F}}  & \boldsymbol{O} \\
    \boldsymbol{O} & \Hat{\boldsymbol{F}}
\end{array}\right])\ . 
\end{align}
Note that the following equations hold:
\begin{align}
\label{equation-proof:colloary-case-I-rank-equality-1}
&rank(\left[\begin{array}{cc}
   \Hat{\boldsymbol{F}}  & \boldsymbol{O} \\
    \boldsymbol{O} & \Hat{\boldsymbol{F}}
\end{array}\right])=2\cdot rank(\Hat{\boldsymbol{F}})\ ,\\
\label{equation-proof:colloary-case-I-rank-equality-2}
&rank(\left[\begin{array}{cc}
   \Hat{\boldsymbol{F}}  & diag(\boldsymbol{p}^{-1})\Hat{\boldsymbol{Z}}^{*}\boldsymbol{V}^{-1} \\
    \boldsymbol{O} & \Hat{\boldsymbol{F}}
\end{array}\right])\geq rank(diag(\boldsymbol{p}^{-1})\Hat{\boldsymbol{Z}}^{*}\boldsymbol{V}^{-1})=rank(\Hat{\boldsymbol{Z}}^{*})\ .
\end{align}
Thus, we can construct special cases of $\Hat{\boldsymbol{F}}$, $\Hat{\boldsymbol{Z}}^{*}$ with the following condition:
\begin{align}
\label{equation-proof:colloary-case-I-fail-2}
rank(\Hat{\boldsymbol{Z}}^{*})>2\cdot rank(\Hat{\boldsymbol{F}})\ .
\end{align}
Those cases make Eq.~\ref{equation-proof:colloary-case-I-fail-1} hold for all values of $\boldsymbol{p}^{-1}$, $\boldsymbol{V}^{-1}$, thus leading to inexistence of $\boldsymbol{p}^{-1}$, $\boldsymbol{V}^{-1}$ making two matrices in Eq.~\ref{equation-proof:colloary-case-I-condition} equivalent. 
That is, there is \textbf{no} solution (value of parameter) making $\Hat{\boldsymbol{Z}}^{*}$ can be constructed, indicating $(\mathcal{T}^{(2)}+\mathcal{T}^{(2)})$ fails to construct target output $\Hat{\boldsymbol{Z}}^{*}$. 
We thereby prove the failure cases exist on $(\mathcal{T}^{(2)}+\mathcal{T}^{(2)})$. 

\par {\bf \ding{173}-For $(\mathcal{T}^{(3)}+\mathcal{T}^{(3)})$:} 
\par Based on Eq.~\ref{equation-proof:colloary-paradigm-3}, the target output construction with $(\mathcal{T}^{(3)}+\mathcal{T}^{(3)})$ can be formulated as follows: 
\begin{align}
\label{equation-proof:colloary-case-II-1}
\Hat{\boldsymbol{Z}}^{*}&=\left[diag(\boldsymbol{p}^{(1)})\Hat{\boldsymbol{F}}_{:1},diag(\boldsymbol{p}^{(2)})\Hat{\boldsymbol{F}}_{:2},...,diag(\boldsymbol{p}^{(C)})\Hat{\boldsymbol{F}}_{:C}\right]+\left[diag(\boldsymbol{q}^{(1)})\Hat{\boldsymbol{F}}_{:1},diag(\boldsymbol{q}^{(2)})\Hat{\boldsymbol{F}}_{:2},...,diag(\boldsymbol{q}^{(C)})\Hat{\boldsymbol{F}}_{:C}\right]\ , \notag\\
&=\left[diag(\boldsymbol{p}^{(1)}+\boldsymbol{q}^{(1)})\Hat{\boldsymbol{F}}_{:1},diag(\boldsymbol{p}^{(2)}+\boldsymbol{q}^{(2)})\Hat{\boldsymbol{F}}_{:2},...,diag(\boldsymbol{p}^{(C)}+\boldsymbol{q}^{(C)})\Hat{\boldsymbol{F}}_{:C}\right]\ .
\end{align}
Therefore, with taking $\boldsymbol{p}^{(c)}+\boldsymbol{q}^{(c)}\doteq\boldsymbol{w}^{(c)}$, $c=1,2,...,C$, the $(\mathcal{T}^{(3)}+\mathcal{T}^{(3)})$ is equivalent to a basic $\mathcal{T}^{(3)}$, thus keeping the same failure cases. 

\par {\bf \ding{174}-For $(\mathcal{T}^{(1)}+\mathcal{T}^{(2)}+\mathcal{T}^{(3)})$:} 
\par To begin with, we introduce a lemma formulated as below:
\begin{lemma}
\label{appendix-lemma-zero-space-combine}
Suppose two $C$-dimensional vectors $\boldsymbol{x}$ and $\boldsymbol{y}$ satisfying:
\begin{align}
\label{equation-proof:colloary-case-III-1}
&\text{Condition 1:}\ \text{Both $\boldsymbol{x}_{1}$ and $\boldsymbol{y}_{1}$ are \textbf{non-zero} values}\ , \\
&\text{Condition 2:}\ \frac{\boldsymbol{x}_{2}}{\boldsymbol{x}_{1}}\neq\frac{\boldsymbol{y}_{2}}{\boldsymbol{y}_{1}}\ .
\end{align}
Thus, the null spaces $\mathcal{U}\doteq\left\{\boldsymbol{u}|\langle\boldsymbol{x},\boldsymbol{u}\rangle=0\right\}$ and $\mathcal{V}\doteq\left\{\boldsymbol{v}|\langle\boldsymbol{y},\boldsymbol{v}\rangle=0\right\}$ satisfy $(\mathcal{U}\bigcup\mathcal{V})=\mathbb{R}^{C}$.
\end{lemma}
\begin{proof}
\label{appendix-proof:lemma-zero-space-combine}
Based on definitions of $\mathcal{U}$ and $\mathcal{V}$, we can show (one of) the bases of $\mathcal{U}$, $\mathcal{V}$ as follows:
\begin{align}
\label{equation-proof:colloary-case-III-2-1}
&\text{Bases of}\,\ \mathcal{U}=(
\left[\begin{array}{c}
    -\frac{\boldsymbol{x}_{2}}{\boldsymbol{x}_{1}} \\
     1 \\
     0 \\
     \vdots \\
     0
\end{array}\right],
\left[\begin{array}{c}
    -\frac{\boldsymbol{x}_{3}}{\boldsymbol{x}_{1}} \\
     0 \\
     1 \\
     \vdots \\
     0
\end{array}\right],...,
\left[\begin{array}{c}
    -\frac{\boldsymbol{x}_{C}}{\boldsymbol{x}_{1}} \\
     0 \\
     0 \\
     \vdots \\
     1
\end{array}\right]
)\doteq(\boldsymbol{u}^{(1)},\boldsymbol{u}^{(2)},...,\boldsymbol{u}^{(C-1)})\ , \\
\label{equation-proof:colloary-case-III-2-2}
&\text{Bases of}\,\ \mathcal{V}=(
\left[\begin{array}{c}
    -\frac{\boldsymbol{y}_{2}}{\boldsymbol{y}_{1}} \\
     1 \\
     0 \\
     \vdots \\
     0
\end{array}\right],
\left[\begin{array}{c}
    -\frac{\boldsymbol{y}_{3}}{\boldsymbol{y}_{1}} \\
     0 \\
     1 \\
     \vdots \\
     0
\end{array}\right],...,
\left[\begin{array}{c}
    -\frac{\boldsymbol{y}_{C}}{\boldsymbol{y}_{1}} \\
     0 \\
     0 \\
     \vdots \\
     1
\end{array}\right]
)\doteq(\boldsymbol{v}^{(1)},\boldsymbol{v}^{(2)},...,\boldsymbol{v}^{(C-1)})\ .
\end{align}
Next, we show the vector set $(\boldsymbol{u}^{(1)},\boldsymbol{u}^{(2)},...,\boldsymbol{u}^{(C-1)},\boldsymbol{v}^{(1)})$ is bases of $(\mathcal{U}\bigcup\mathcal{V})$ through proving by contradiction. 
Suppose $\boldsymbol{v}^{(1)})$ can be represented by linear combinations of $(\boldsymbol{u}^{(1)},\boldsymbol{u}^{(2)},...,\boldsymbol{u}^{(C-1)})$, we have the following equation:
\begin{align}
\label{equation-proof:colloary-case-III-3}
t_{1}\boldsymbol{u}^{(1)}+t_{2}\boldsymbol{u}^{(2)}+...+t_{C-1}\boldsymbol{u}^{(C-1)}=\boldsymbol{v}^{(1)}\ ,
\end{align}
where $t_{1},t_{2},...,t_{C-1}\in\mathbb{R}$ are not all $0$. 
Considering Eq.~\ref{equation-proof:colloary-case-III-2-1} and Eq.~\ref{equation-proof:colloary-case-III-2-1}, we know Eq.~\ref{equation-proof:colloary-case-III-3} indicates the following relations:
\begin{align}
\label{equation-proof:colloary-case-III-4}
\left[\begin{array}{c}
    -\frac{\sum_{i=1}^{C-1}t_{i}\boldsymbol{x}_{i+1}}{\boldsymbol{x}_{1}} \\
     t_{1} \\
     t_{2} \\
     \vdots \\
     t_{C-1}
\end{array}\right]=\left[\begin{array}{c}
    -\frac{\boldsymbol{y}_{2}}{\boldsymbol{y}_{1}} \\
     1 \\
     0 \\
     \vdots \\
     0
\end{array}\right]\Longrightarrow\left\{\begin{array}{c}
     \frac{\sum_{i=1}^{C-1}t_{i}\boldsymbol{x}_{i+1}}{\boldsymbol{x}_{1}}=\frac{\boldsymbol{y}_{2}}{\boldsymbol{y}_{1}}  \\
     t_{1}=1 \\
     t_{2}=0 \\
     \vdots \\
     t_{C-1}=0 \\
\end{array}\right.\Longrightarrow\frac{\boldsymbol{x}_{2}}{\boldsymbol{x}_{1}}=\frac{\boldsymbol{y}_{2}}{\boldsymbol{y}_{1}}\ ,
\end{align}
which contradicts to our conditions on $\boldsymbol{x}$ and $\boldsymbol{y}$. 
Thus, based on the contradiction, we prove that $\boldsymbol{v}^{(1)})$ can \textbf{NOT} be represented by linear combinations of $(\boldsymbol{u}^{(1)},\boldsymbol{u}^{(2)},...,\boldsymbol{u}^{(C-1)})$, indicating $(\boldsymbol{u}^{(1)},\boldsymbol{u}^{(2)},...,\boldsymbol{u}^{(C-1)},\boldsymbol{v}^{(1)})$ is a set of vectors being linearly independent in $(\mathcal{U}\bigcup\mathcal{V})$. 
This means
\begin{equation}
\label{equation-proof:colloary-case-III-5}
\dim(\mathcal{U}\bigcup\mathcal{V})\geq C\ .
\end{equation}
On the other hand, we know 
\begin{equation}
\label{equation-proof:colloary-case-III-6}
\dim(\mathcal{U}\bigcup\mathcal{V})\leq C\ .
\end{equation}
Thus, we show that 
\begin{equation}
\label{equation-proof:colloary-case-III-7}
\dim(\mathcal{U}\bigcup\mathcal{V})=C=\dim(\mathbb{R}^{C})\ \Longleftrightarrow \ (\mathcal{U}\bigcup\mathcal{V})=\mathbb{R}^{C}\ , 
\end{equation}
and the proof of Lemma~\ref{appendix-lemma-zero-space-combine} is completed. 
\end{proof}
Now we can begin to investigate the target output construction with $(\mathcal{T}^{(1)}+\mathcal{T}^{(2)}+\mathcal{T}^{(3)})$. 
Based on Eq.~\ref{equation-proof:colloary-paradigm-1},~\ref{equation-proof:colloary-paradigm-2} and~\ref{equation-proof:colloary-paradigm-3}, the $c$-th column of $\Hat{\boldsymbol{Z}}^{*}$ achieved with $(\mathcal{T}^{(1)}+\mathcal{T}^{(2)}+\mathcal{T}^{(3)})$ can be formulated as follows: 
\begin{align}
\label{equation-proof:colloary-case-III-8}
\Hat{\boldsymbol{Z}}^{*}_{:c}=diag(\boldsymbol{g})\Hat{\boldsymbol{F}}_{:c} + diag(\boldsymbol{q})\Hat{\boldsymbol{F}}\boldsymbol{W}_{:c} + diag(\boldsymbol{p}^{(c)})\Hat{\boldsymbol{F}}_{:c}\ .
\end{align}
Thus, suppose arbitrary two row numbers $m$ and $n$, where $m,n\in\{1,2,...,N\}$, $m\neq n$, the corresponding elements of $\Hat{\boldsymbol{Z}}^{*}_{:c}$, i.e., $\Hat{\boldsymbol{Z}}^{*}_{mc}$ and $\Hat{\boldsymbol{Z}}^{*}_{nc}$, are shown as below:
\begin{align}
\label{equation-proof:colloary-case-III-9-1}
&\Hat{\boldsymbol{Z}}^{*}_{mc}=(\boldsymbol{g}_{m}+\boldsymbol{p}^{(c)}_{m})\Hat{\boldsymbol{F}}_{mc} + \boldsymbol{q}_{m}\cdot\langle\Hat{\boldsymbol{F}}_{m:}\boldsymbol{W}_{:c}\rangle\ , \\
\label{equation-proof:colloary-case-III-9-2}
&\Hat{\boldsymbol{Z}}^{*}_{nc}=(\boldsymbol{g}_{n}+\boldsymbol{p}^{(c)}_{n})\Hat{\boldsymbol{F}}_{nc} + \boldsymbol{q}_{n}\cdot\langle\Hat{\boldsymbol{F}}_{n:}\boldsymbol{W}_{:c}\rangle\ .
\end{align}
On the other hand, we construct $\Hat{\boldsymbol{Z}}^{*}_{mc}$, $\Hat{\boldsymbol{Z}}^{*}_{nc}$, $\Hat{\boldsymbol{F}}_{m:}$, $\Hat{\boldsymbol{F}}_{n:}$ satisfying the following conditions:
\begin{align}
\label{equation-proof:colloary-case-III-10-1}
&\Hat{\boldsymbol{F}}_{m1}, \Hat{\boldsymbol{F}}_{n1}, \Hat{\boldsymbol{Z}}^{*}_{mc}, \Hat{\boldsymbol{Z}}^{*}_{nc}\ \text{are all \textbf{non-zero} values}\ ,\\
\label{equation-proof:colloary-case-III-10-2}
&\frac{\Hat{\boldsymbol{F}}_{m2}}{\Hat{\boldsymbol{F}}_{m1}}\neq\frac{\Hat{\boldsymbol{F}}_{n2}}{\Hat{\boldsymbol{F}}_{n1}}\ , \\
\label{equation-proof:colloary-case-III-10-3}
&\Hat{\boldsymbol{F}}_{mc}=\Hat{\boldsymbol{F}}_{nc}=0\ .
\end{align}
According to our Lemma~\ref{appendix-lemma-zero-space-combine}, conditions of Eq.~\ref{equation-proof:colloary-case-III-10-1} and Eq.~\ref{equation-proof:colloary-case-III-10-2} make the union set of the null spaces of $\Hat{\boldsymbol{F}}_{m:}$ and $\Hat{\boldsymbol{F}}_{n:}$ be the $\mathbb{R}^{C}$, leading to
\begin{equation}
\label{equation-proof:colloary-case-III-11}
\boldsymbol{W}_{:c}\in\mathbb{R}^{C}=\left\{\boldsymbol{u}|\langle\Hat{\boldsymbol{F}}_{m:},\boldsymbol{u}\rangle=0\right\}\bigcup\left\{\boldsymbol{v}|\langle\Hat{\boldsymbol{F}}_{n:},\boldsymbol{v}\rangle=0\right\}\ .  
\end{equation}
Eq.~\ref{equation-proof:colloary-case-III-11} indicates \textbf{at least one} of the following equations holds: 
\begin{align}
\label{equation-proof:colloary-case-III-12-1}
&\langle\Hat{\boldsymbol{F}}_{m:}\boldsymbol{W}_{:c}\rangle=0\ ,\\
\label{equation-proof:colloary-case-III-12-2}
&\langle\Hat{\boldsymbol{F}}_{n:}\boldsymbol{W}_{:c}\rangle=0\ . 
\end{align}
Without loss of generality, we suppose Eq.~\ref{equation-proof:colloary-case-III-12-1} holds. 
Thus, with further considering the condition as Eq.~\ref{equation-proof:colloary-case-III-10-3}, the right side of Eq.~\ref{equation-proof:colloary-case-III-9-1} is reformulated as
\begin{align}
\label{equation-proof:colloary-case-III-13}
&(\boldsymbol{g}_{m}+\boldsymbol{p}^{(c)}_{m})\Hat{\boldsymbol{F}}_{mc} + \boldsymbol{q}_{m}\cdot\langle\Hat{\boldsymbol{F}}_{m:}\boldsymbol{W}_{:c}\rangle\ ,\notag\\
=&(\boldsymbol{g}_{m}+\boldsymbol{p}^{(c)}_{m})\cdot0 + \boldsymbol{q}_{m}\cdot0\ ,\notag\\
=&0\ , 
\end{align}
which holds for all values of $\boldsymbol{g}_{m}$, $\boldsymbol{p}^{(c)}_{m}$ and $\boldsymbol{q}_{m}$. 
However, $\Hat{\boldsymbol{Z}}^{*}_{mc}$ is constructed to be a non-zero value, indicating this element can \textbf{never} be constructed for all parameter values. 
Thus, we prove that in the cases with conditions as Eq.~\ref{equation-proof:colloary-case-III-10-1},~\ref{equation-proof:colloary-case-III-10-2} and~\ref{equation-proof:colloary-case-III-10-3}, the $(\mathcal{T}^{(1)}+\mathcal{T}^{(2)}+\mathcal{T}^{(3)})$ would fail to construct the target output $\Hat{\boldsymbol{Z}}^{*}$. 
\par Hence, based on the proof of \ding{172}, \ding{173}, and \ding{174}, we show that simply adding the paradigms cannot tackle the failure cases but still remain issues. 
\end{proof}

\subsection{Equivalence between Eq.~\ref{equation-2dconv} and Eq.~\ref{equation-2stepsconv}}
\label{appendix-proofs-and-derivations-2dconv-to-2stepsconv}
In this section, we provide derivation to prove the equivalence between Eq.~\ref{equation-2dconv} and Eq.~\ref{equation-2stepsconv}. 
\begin{proof}
\label{appendix-proofs-and-derivations-2dconv-to-2stepsconv-derivation}
We consider proving the equivalence on each element of Eq.~\ref{equation-2dconv} and Eq.~\ref{equation-2stepsconv}. 
Specifically, for the $k$-th element of $\textit{Vec}(\boldsymbol{Z})_{k}$ ($k\in\mathbb{N}^{+}, 1\leq k\leq NC$), we apply decomposition to the index as follows: 
\begin{equation}
\label{equation-id-decomposition}
k=i+N(j-1)\ ,\quad i, j\in\mathbb{N}^{+},\quad 1\leq i,j\leq C\ .
\end{equation}
Such decomposition demonstrates $\textit{Vec}(\boldsymbol{Z})_{k=i+N(j-1)}=\boldsymbol{Z}_{ij}$, and also leads to derivation on Eq.~\ref{equation-2stepsconv} as follows: 
\begin{align}
\label{equation-2stepsconv-to-2dconv-1}
\boldsymbol{Z}_{ij}=&\textit{Vec}(\boldsymbol{Z})_{k=i+N(j-1)}\ ,\notag\\
=&\left[\boldsymbol{\Phi}^{(1,j)}, \boldsymbol{\Phi}^{(2,j)},...,\boldsymbol{\Phi}^{(C,j)}\right]_{i:}\cdot \textit{Vec}\left(\boldsymbol{F}\right)\ ,\notag\\
=&\left[\boldsymbol{\Phi}^{(1,j)}_{i:}, \boldsymbol{\Phi}^{(2,j)}_{i:},...,\boldsymbol{\Phi}^{(C,j)}_{i:}\right]\cdot \textit{Vec}\left(\boldsymbol{F}\right)\ ,\notag\\
=&\sum_{c=1}^{C}\boldsymbol{\Phi}^{(c,j)}_{i:}\cdot \textit{Vec}\left(\boldsymbol{F}\right)_{N(c-1)+1\sim Nc}\ ,\notag\\
(\text{similarly,}\ \textit{Vec}(\boldsymbol{F})_{k=i+N(j-1)}=\boldsymbol{F}_{ij})\ =&\sum_{n=1}^{N}\sum_{c=1}^{C}\boldsymbol{\Phi}^{(c,j)}_{in}\boldsymbol{F}_{nc}\ .  
\end{align}
Since each $\boldsymbol{\Phi}^{(c,j)}_{in}$ is unique regarding the index $\{c, j, i, n\}$, with comparison between Eq.~\ref{equation-2dconv} and Eq.~\ref{equation-2stepsconv-to-2dconv-1}, we find a \textit{surjective} mapping that maps $\boldsymbol{\Phi}^{(c,j)}_{in}$ to $\boldsymbol{\Omega}_{nc}^{(i,j)}$ and satisfies $(\boldsymbol{\Phi}^{(c,j)}_{in}=\boldsymbol{\Omega}_{nc}^{(i,j)})$. 
Furthermore, as both the number of $\boldsymbol{\Phi}^{(c,j)}_{in}$ and $\boldsymbol{\Omega}_{nc}^{(i,j)}$ are $N\times N \times C \times C=N^{2}C^{2}$, the \textit{surjective} mapping becomes \textit{bijective} mapping, indicating $N^{2}C^{2}$ unique mapping pairs $(\boldsymbol{\Phi}^{(c,j)}_{in}, \boldsymbol{\Omega}_{nc}^{(i,j)})$ that satisfy $\boldsymbol{\Phi}^{(c,j)}_{in}=\boldsymbol{\Omega}_{nc}^{(i,j)}$. 
Thus, the equivalence between Eq.~\ref{equation-2dconv} and Eq.~\ref{equation-2stepsconv} is proved. 
\end{proof}

\subsection{Proof of Proposition~\ref{proposition-generalized-to-special}}
\label{appendix-proofs-and-derivations-proposition-generalized-to-special}
\begin{proof}
\label{appendix-proofs-and-derivations-proof:proposition-generalized-to-special}
In this proof, we show how 2-D graph convolution can perform other graph convolution paradigms in the order of Paradigm {\bf (I)}, {\bf (II)} and {\bf (III)}. 
\par {\bf \ding{172}-Performing Paradigm (I):} 
\par We set $\boldsymbol{\Phi}^{(c,j)}_{\mathcal{G}}$ with the following constraints: 
\begin{align}
\label{equation-constraints-to-single-filter-no-feature-interaction-1}
&\boldsymbol{\Phi}^{(c,c)}_{\mathcal{G}}=\boldsymbol{\Phi}_{\mathcal{G}}\ ,\quad  c=1,2,...,C\  , \\
&\boldsymbol{\Phi}^{(c,j)}_{\mathcal{G}}=\boldsymbol{O}\ ,\quad  c,j=1,2,...,C\ ,\quad c\neq j\ .
\end{align}
Thus, the 2-D graph convolution with the above constraints is formulated as: 
\begin{align}
\label{equation-constraints-to-single-filter-no-feature-interaction-2}
\boldsymbol{Z}=&\textit{Vec}^{-1}\left(
\left[\begin{array}{ccc}
    \boldsymbol{\Phi}_{\mathcal{G}} & \cdots & \boldsymbol{O} \\
    \vdots & \ddots & \vdots \\
    \boldsymbol{O} & \cdots & \boldsymbol{\Phi}_{\mathcal{G}}
\end{array}\right]
\cdot \textit{Vec}\left(\boldsymbol{F}\right)\right)\notag=\boldsymbol{\Phi}_{\mathcal{G}}\boldsymbol{F},\\
\overset{\text{Column scale}}{\Longleftrightarrow}\quad \boldsymbol{Z}_{:c}=&\boldsymbol{\Phi}_{\mathcal{G}}\boldsymbol{F}_{:c}\ ,\quad c=1,2,...,C\ ,
\end{align}
which is equivalent to Eq.~\ref{equation-paradigm-single-filter-and-no-feature-interaction}, i.e., Paradigm {\bf (I)}. 
\par {\bf \ding{173}-Performing Paradigm (II):} 
\par We set $\boldsymbol{\Phi}^{(c,j)}_{\mathcal{G}}$ with the following constraints: 
\begin{align}
\label{equation-constraints-to-single-filter-and-feature-interaction-1}
\boldsymbol{\Phi}^{(c,j)}_{\mathcal{G}}=r_{j,c}\boldsymbol{\Phi}_{\mathcal{G}},\quad r_{j,c}\in\mathbb{R},\quad  c,j=1,2,...,C\ .
\end{align}
Thus, the 2-D graph convolution with the above constraints is formulated as: 
\begin{align}
\label{equation-constraints-to-single-filter-and-feature-interaction-2}
\boldsymbol{Z}=&\textit{Vec}^{-1}\left(
\left[\begin{array}{ccc}
    r_{1,1}\boldsymbol{\Phi}_{\mathcal{G}} & \cdots & r_{C,1}\boldsymbol{\Phi}_{\mathcal{G}} \\
    \vdots & \ddots & \vdots \\
    r_{1,C}\boldsymbol{\Phi}_{\mathcal{G}} & \cdots & r_{C,C}\boldsymbol{\Phi}_{\mathcal{G}}
\end{array}\right]
\cdot \textit{Vec}\left(\boldsymbol{F}\right)\right), \\
\overset{\text{Column scale}}{\Longleftrightarrow}\quad \boldsymbol{Z}_{:c}=&\sum_{j=1}^{C}r_{j,c}\boldsymbol{\Phi}_{\mathcal{G}}\boldsymbol{F}_{:j}\ ,\quad c=1,2,...,C\ .
\end{align}
With taking $r_{j,c}=\boldsymbol{R}_{jc}$, the above equation is equivalent to Eq.~\ref{equation-paradigm-single-filter-and-feature-interaction}, i.e., Paradigm {\bf (II)}. 
\par {\bf \ding{174}-Performing Paradigm (III):} 
\par We set $\boldsymbol{\Phi}^{(c,j)}_{\mathcal{G}}$ with the following constraints: 
\begin{align}
\label{equation-constraints-to-individual-filter-and-no-feature-interaction-1}
&\boldsymbol{\Phi}^{(c,c)}_{\mathcal{G}}=\boldsymbol{\Phi}^{(c)}_{\mathcal{G}}\ ,\quad  c=1,2,...,C\  , \\
&\boldsymbol{\Phi}^{(c,j)}_{\mathcal{G}}=\boldsymbol{O}\ ,\quad  c,j=1,2,...,C\ ,\quad c\neq j\ .
\end{align}
Thus, the 2-D graph convolution with the above constraints is formulated as: 
\begin{align}
\label{equation-constraints-to-individual-filter-and-no-feature-interaction-2}
\boldsymbol{Z}=&\textit{Vec}^{-1}\left(
\left[\begin{array}{ccc}
    \boldsymbol{\Phi}^{(1)}_{\mathcal{G}} & \cdots & \boldsymbol{O} \\
    \vdots & \ddots & \vdots \\
    \boldsymbol{O} & \cdots & \boldsymbol{\Phi}^{(C)}_{\mathcal{G}}
\end{array}\right]
\cdot \textit{Vec}\left(\boldsymbol{F}\right)\right), \\
\overset{\text{Column scale}}{\Longleftrightarrow}\quad \boldsymbol{Z}_{:c}=&\boldsymbol{\Phi}^{(c)}_{\mathcal{G}}\boldsymbol{F}_{:c}\ ,\quad c=1,2,...,C\ ,
\end{align}
which is equivalent to Eq.~\ref{equation-paradigm-individual-filter-and-no-feature-interaction}, i.e., Paradigm {\bf (III)}. 
\par Based on the derivations above, we prove that the 2-D graph convolution can perform other convolution paradigms with setting $\boldsymbol{\Phi}^{(c,j)}_{\mathcal{G}}$ to be special cases, indicating our 2-D graph convolution is a generalized framework of formal spectral graph convolution. 
\end{proof}

\subsection{Proof of Theorem~\ref{theorem-constructing-target-output-embeddings}}
\label{appendix-proofs-and-derivations-theorem-constructing-target-output-embeddings}
\begin{proof}
\label{appendix-proofs-and-derivations-proof:theorem-constructing-target-output-embeddings}
To begin with, note the graph convolution operator $\boldsymbol{\Phi}_{\mathcal{G}}$ can be decomposed as $\boldsymbol{\Phi}_{\mathcal{G}}=\boldsymbol{U}diag(\boldsymbol{g})\boldsymbol{U}^{T}$, we can thereby investigate an equivalent statement on frequency domain, i.e., constructing target $\Hat{\boldsymbol{Z}}^{*}=\boldsymbol{U}^{T}\boldsymbol{Z}^{*}$ with given graph signal $\Hat{\boldsymbol{F}}=\boldsymbol{U}^{T}\boldsymbol{F}$. 
Next, we prove the Theorem~\ref{theorem-constructing-target-output-embeddings} from the perspective of element construction of $\Hat{\boldsymbol{Z}}^{*}$. 
We first derive the 2-D graph convolution as Eq.~\ref{equation-2dgraphconv} to be the following formulation:
\begin{align}
\label{equation-derived-2dgraphconv}
\boldsymbol{U}^{T}\boldsymbol{Z}=&\textit{Vec}^{-1}\left(
\left[\begin{array}{ccc}
    diag(\boldsymbol{g}^{(1,1)}) & \cdots & diag(\boldsymbol{g}^{(C,1)}) \\
    \vdots & \ddots & \vdots \\
    diag(\boldsymbol{g}^{(1,C)}) & \cdots & diag(\boldsymbol{g}^{(C,C)})
\end{array}\right]
\cdot \textit{Vec}\left(\boldsymbol{U}^{T}\boldsymbol{F}\right)\right), 
\end{align}
where $\boldsymbol{g}^{(c,j)}\in\mathbb{R}^{N}$ is the parameterized graph filter. 
With the above formulation, proving the statement is thereby equivalent to proving the following relations hold:
\begin{align}
\label{equation-prove-target-1}
\Hat{\boldsymbol{Z}}^{*}_{:c}=&\sum_{j=1}^{C}diag(\boldsymbol{g}^{(c,j)})\Hat{\boldsymbol{F}}_{:j}\ ,\quad c=1,2,...,C\ ,\\
\label{equation-prove-target-2}
\overset{\text{Element scale}}{\Longleftrightarrow}\quad \Hat{\boldsymbol{Z}}^{*}_{nc}=&\sum_{j=1}^{C}\boldsymbol{g}^{(c,j)}_{n}\Hat{\boldsymbol{F}}_{nj}\ ,\quad c=1,2,...,C\ ,
\end{align}
where $\boldsymbol{g}^{(c,j)}_{n}$ denotes the $n$-th element of filter $\boldsymbol{g}^{(c,j)}$. 
\par Note that $\boldsymbol{F}$ is nontrivial on frequency, i.e., no zero row exists in $\boldsymbol{U}^{T}\boldsymbol{F}$, the $\Hat{\boldsymbol{F}}_{nj}$ ($j=1,2,...,C$) are not all $0$. 
On the other hand, $\boldsymbol{g}^{(c,j)}_{n}$ is unique towards $\Hat{\boldsymbol{F}}_{nj}$ in constructing $\Hat{\boldsymbol{Z}}^{*}_{nc}$, indicating the independence of parameter selection of $\boldsymbol{g}^{(c,j)}_{n}$. 
Therefore, with considering the parameter $\boldsymbol{g}^{(c,j)}_{n}\in\mathbb{R}$, similar to solving linear equations, we can always choose appropriate $\{\boldsymbol{g}^{(c,1)}_{n},\boldsymbol{g}^{(c,2)}_{n},...,\boldsymbol{g}^{(c,C)}_{n}\}$ to make Eq.~\ref{equation-prove-target-2} hold (and further make Eq.~\ref{equation-prove-target-1} hold), indicating that 2-D graph convolution can always construct all elements of $\boldsymbol{Z}^{*}$ with $0$ construction error. 
Furthermore, note that the number of parameter values satisfying both Eq.~\ref{equation-prove-target-1} and Eq.~\ref{equation-prove-target-2} are at least one. 
Therefore, for the parameter set of 2-D graph convolution, i.e., $\Bar{\boldsymbol{\Psi}}$, we can always find at least one appropriate parameter value $\Bar{\boldsymbol{\psi}}$, making $0$ construction error hold. 
\end{proof}

\subsection{Proof of Theorem~\ref{theorem-2dgraphconv-irreducible-number-parameters}}
\label{appendix-proofs-and-derivations-theorem-2dgraphconv-irreducible-number-parameters}
\begin{proof}
\label{appendix-proofs-and-derivations-proof:theorem-2dgraphconv-irreducible-number-parameters}
We provide constructive proof to Theorem~\ref{theorem-2dgraphconv-irreducible-number-parameters}. 
To begin with, following the proof of Theorem~\ref{theorem-constructing-target-output-embeddings} above, we consider the element-wise construction of $\Hat{\boldsymbol{Z}}^{*}$ with 2-D graph convolution as below:
\begin{align}
\label{equation-element-wise-target-construction-2dgraphconv-1}
\Hat{\boldsymbol{Z}}^{*}_{nc}=&\sum_{j=1}^{C}\boldsymbol{g}^{(c,j)}_{n}\Hat{\boldsymbol{F}}_{nj}\ ,\quad c=1,2,...,C\ ,
\end{align}
where $\boldsymbol{g}^{(c,j)}_{n}$ denotes the $n$-th element of filter $\boldsymbol{g}^{(c,j)}$. 
Next, we divide the proof into two parts based on the mentioned two types of parameter reduction on $\boldsymbol{\Psi}$: \ding{172}-setting $\boldsymbol{\Psi}\setminus\boldsymbol{\Psi}_{\text{sub}}$ to be constant; \ding{173}-sharing $\boldsymbol{\Psi}_{\text{sub}}$ to $\boldsymbol{\Psi}\setminus\boldsymbol{\Psi}_{\text{sub}}$. 
\par {\bf \ding{172}-Setting $\boldsymbol{\Psi}\setminus\boldsymbol{\Psi}_{\text{sub}}$ to be constant:} 
\par Suppose $\boldsymbol{g}^{(c,j)}_{n}\in\boldsymbol{\Psi}\setminus\boldsymbol{\Psi}_{\text{sub}}$ is set to be constant $a\in\mathbb{R}$. 
Accordingly, Eq.~\ref{equation-element-wise-target-construction-2dgraphconv-1} can be rewritten as follows:
\begin{align}
\label{equation-element-wise-target-construction-2dgraphconv-2-1}
\Hat{\boldsymbol{Z}}^{*}_{nc}=&\underbrace{a\Hat{\boldsymbol{F}}_{nj}}_{\text{constant}}+\sum_{\substack{1\leq i\leq C\\i\neq j}}\boldsymbol{g}^{(c,i)}_{n}\Hat{\boldsymbol{F}}_{ni}\ .
\end{align}
Based on the equation above, we can construct $\Hat{\boldsymbol{F}}$, $\Hat{\boldsymbol{Z}}^{*}$ satisfying the following conditions:
\begin{align}
\label{equation-element-wise-target-construction-2dgraphconv-2-2}
&\Hat{\boldsymbol{F}}_{ni}=0\ ,\quad 1\leq i\leq C\ ,\quad i\neq j \ ,\notag\\
&\Hat{\boldsymbol{F}}_{nj}\neq0 \ ,\notag\\
&\Hat{\boldsymbol{Z}}^{*}_{nc}\neq a\Hat{\boldsymbol{F}}_{nj}\ .
\end{align}
With the conditions above, Eq.~\ref{equation-element-wise-target-construction-2dgraphconv-1} \textbf{never} hold for all parameter values, indicating that failure cases exist in the 2-D graph convolution with partial constant parameters, and the proof of (I) is finished. 

\par {\bf \ding{173}-Sharing $\boldsymbol{\Psi}_{\text{sub}}$ to $\boldsymbol{\Psi}\setminus\boldsymbol{\Psi}_{\text{sub}}$:} 
\par Suppose the value of $\boldsymbol{g}^{(c,j)}_{n}\in\boldsymbol{\Psi}\setminus\boldsymbol{\Psi}_{\text{sub}}$ is shared by $\boldsymbol{g}^{(c^{'},j^{'})}_{n^{'}}\in\boldsymbol{\Psi}_{\text{sub}}$, i.e., $\boldsymbol{g}^{(c,j)}_{n}\equiv\boldsymbol{g}^{(c^{'},j^{'})}_{n^{'}}$. 
Accordingly, we have seven different cases, i.e., $(n\neq n^{'},c\neq c^{'},j\neq j^{'})$, $(n\neq n^{'},c\neq c^{'},j=j^{'})$, $(n\neq n^{'},c=c^{'},j\neq j^{'})$, $(n\neq n^{'},c=c^{'},j=j^{'})$, $(n=n^{'},c\neq c^{'},j\neq j^{'})$, $(n=n^{'},c\neq c^{'},j=j^{'})$ and $(n=n^{'},c=c^{'},j\neq j^{'})$. 
Because the proofs of those cases are similar, for convenience, we provide the proof for the most complicated case $(n\neq n^{'},c\neq c^{'},j\neq j^{'})$ here. 
Specifically, we can construct $\Hat{\boldsymbol{F}}$, $\Hat{\boldsymbol{Z}}^{*}$ satisfying the following conditions:
\begin{align}
\label{equation-element-wise-target-construction-2dgraphconv-3-1-1}
&\Hat{\boldsymbol{F}}_{ni}=0\ ,\quad 1\leq i\leq C\ ,\quad i\neq j \ ,\notag\\
&\Hat{\boldsymbol{F}}_{nj}\neq0 \ ,\notag\\
&\Hat{\boldsymbol{F}}_{n^{'}i^{'}}=0\ ,\quad 1\leq i^{'}\leq C\ ,\quad i^{'}\neq j^{'} \ ,\notag\\
&\Hat{\boldsymbol{F}}_{n^{'}j^{'}}\neq0 \ ,\notag\\
&\frac{\Hat{\boldsymbol{Z}}^{*}_{nc}}{\Hat{\boldsymbol{F}}_{nj}}\neq\frac{\Hat{\boldsymbol{Z}}^{*}_{n^{'}c^{'}}}{\Hat{\boldsymbol{F}}_{n^{'}j^{'}}}\ .
\end{align}
With the conditions above, achieving error-free construction implies the following equations hold:
\begin{align}
\label{equation-element-wise-target-construction-2dgraphconv-3-1-2}
\frac{\Hat{\boldsymbol{Z}}^{*}_{nc}}{\Hat{\boldsymbol{F}}_{nj}}=\boldsymbol{g}^{(c,j)}_{n}\neq\boldsymbol{g}^{(c^{'},j^{'})}_{n^{'}}=\frac{\Hat{\boldsymbol{Z}}^{*}_{n^{'}c^{'}}}{\Hat{\boldsymbol{F}}_{n^{'}j^{'}}}\ . 
\end{align}
which contradicts to the condition of sharing parameter. 
Thus, in these special cases of $\Hat{\boldsymbol{F}}$, $\Hat{\boldsymbol{Z}}^{*}$, the construction error \textbf{never} achieves $0$ for all parameter values. 
As the rest cases can be proved by constructing conditions similar to those in Eq.~\ref{equation-element-wise-target-construction-2dgraphconv-3-1-1}, we complete the proof of (II). 

\par Thus, based on \ding{172} and \ding{173}, we prove that reducing parameters in 2-D graph convolution results in indeed failure cases on target output construction, demonstrating that the 2-D graph convolution is irreducible in the number of parameters. 
\end{proof}

\subsection{Derivation of Obtaining Eq.~\ref{equation-chebyshev-polynomial-2dgraphconv} with Eq.~\ref{equation-2dgraphconv} and Eq.~\ref{equation-chebyshev-interpolation-approximation}}
\label{appendix-proofs-and-derivations-obtaining-chebyshev-polynomial-2dgraphconv}
Note $\sum_{b=0}^{D}\theta_{b}^{(c,j)}T_{d}(x_{b})$ is irrelevant to $b$ but $d$, for simplicity, we use $\Hat{\theta_{d}}^{(c,j)}$ to denote it for each pair of $(c,j)$ and dismiss the constant $\frac{2}{D+1}$. 
Thus, the Eq.~\ref{equation-2dgraphconv} with Eq.~\ref{equation-chebyshev-interpolation-approximation} can be reformulated as follows:
\begin{align*}
\boldsymbol{Z}\doteq&\textit{Vec}^{-1}\left(
\left[\begin{array}{ccc}
    \sum_{d=0}^{D}\Hat{\theta_{d}}^{(1,1)}T_{d}(\Hat{\boldsymbol{L}}) & \cdots & \sum_{d=0}^{D}\Hat{\theta_{d}}^{(C,1)}T_{d}(\Hat{\boldsymbol{L}}) \\
    \vdots & \ddots & \vdots \\
    \sum_{d=0}^{D}\Hat{\theta_{d}}^{(1,C)}T_{d}(\Hat{\boldsymbol{L}}) & \cdots & \sum_{d=0}^{D}\Hat{\theta_{d}}^{(C,C)}T_{d}(\Hat{\boldsymbol{L}})
\end{array}\right]
\cdot \textit{Vec}\left(\boldsymbol{F}\right)\right),\notag\\
=&\left[\sum_{c=1}^{C}\sum_{d=0}^{D}\Hat{\theta_{d}}^{(c,1)}T_{d}(\Hat{\boldsymbol{L}})\boldsymbol{F}_{:c},\sum_{c=1}^{C}\sum_{d=0}^{D}\Hat{\theta_{d}}^{(c,2)}T_{d}(\Hat{\boldsymbol{L}})\boldsymbol{F}_{:c},...,\sum_{c=1}^{C}\sum_{d=0}^{D}\Hat{\theta_{d}}^{(c,C)}T_{d}(\Hat{\boldsymbol{L}})\boldsymbol{F}_{:c}\right],\notag\\
=&\left[\sum_{d=0}^{D}T_{d}(\Hat{\boldsymbol{L}})\left(\sum_{c=1}^{C}\Hat{\theta_{d}}^{(c,1)}\boldsymbol{F}_{:c}\right),\sum_{d=0}^{D}T_{d}(\Hat{\boldsymbol{L}})\left(\sum_{c=1}^{C}\Hat{\theta_{d}}^{(c,2)}\boldsymbol{F}_{:c}\right),...,\sum_{d=0}^{D}T_{d}(\Hat{\boldsymbol{L}})\left(\sum_{c=1}^{C}\Hat{\theta_{d}}^{(c,C)}\boldsymbol{F}_{:c}\right)\right],\notag\\
=&\left[\sum_{d=0}^{D}T_{d}(\Hat{\boldsymbol{L}})\boldsymbol{F}\left[\begin{array}{c}
    \Hat{\theta_{d}}^{(1,1)} \\
    \Hat{\theta_{d}}^{(2,1)} \\
    \vdots \\
    \Hat{\theta_{d}}^{(C,1)}
\end{array}\right],\sum_{d=0}^{D}T_{d}(\Hat{\boldsymbol{L}})\boldsymbol{F}\left[\begin{array}{c}
    \Hat{\theta_{d}}^{(1,2)} \\
    \Hat{\theta_{d}}^{(2,2)} \\
    \vdots \\
    \Hat{\theta_{d}}^{(C,2)}
\end{array}\right],...,\sum_{d=0}^{D}T_{d}(\Hat{\boldsymbol{L}})\boldsymbol{F}\left[\begin{array}{c}
    \Hat{\theta_{d}}^{(1,C)} \\
    \Hat{\theta_{d}}^{(2,C)} \\
    \vdots \\
    \Hat{\theta_{d}}^{(C,C)}
\end{array}\right]\right],\notag\\
=&\left[T_{0}(\Hat{\boldsymbol{L}})\boldsymbol{F}\left[\begin{array}{c}
    \Hat{\theta_{0}}^{(1,1)} \\
    \Hat{\theta_{0}}^{(2,1)} \\
    \vdots \\
    \Hat{\theta_{0}}^{(C,1)}
\end{array}\right],T_{0}(\Hat{\boldsymbol{L}})\boldsymbol{F}\left[\begin{array}{c}
    \Hat{\theta_{0}}^{(1,2)} \\
    \Hat{\theta_{0}}^{(2,2)} \\
    \vdots \\
    \Hat{\theta_{0}}^{(C,2)}
\end{array}\right],...,T_{0}(\Hat{\boldsymbol{L}})\boldsymbol{F}\left[\begin{array}{c}
    \Hat{\theta_{0}}^{(1,C)} \\
    \Hat{\theta_{0}}^{(2,C)} \\
    \vdots \\
    \Hat{\theta_{0}}^{(C,C)}
\end{array}\right]\right]\notag\\
&\quad+\left[T_{1}(\Hat{\boldsymbol{L}})\boldsymbol{F}\left[\begin{array}{c}
    \Hat{\theta_{1}}^{(1,1)} \\
    \Hat{\theta_{1}}^{(2,1)} \\
    \vdots \\
    \Hat{\theta_{1}}^{(C,1)}
\end{array}\right],T_{1}(\Hat{\boldsymbol{L}})\boldsymbol{F}\left[\begin{array}{c}
    \Hat{\theta_{1}}^{(1,2)} \\
    \Hat{\theta_{1}}^{(2,2)} \\
    \vdots \\
    \Hat{\theta_{1}}^{(C,2)}
\end{array}\right],...,T_{1}(\Hat{\boldsymbol{L}})\boldsymbol{F}\left[\begin{array}{c}
    \Hat{\theta_{1}}^{(1,C)} \\
    \Hat{\theta_{1}}^{(2,C)} \\
    \vdots \\
    \Hat{\theta_{1}}^{(C,C)}
\end{array}\right]\right]\notag\\
&\quad\quad\quad\quad\quad\quad\quad\vdots\notag\\
&\quad+\left[T_{D}(\Hat{\boldsymbol{L}})\boldsymbol{F}\left[\begin{array}{c}
    \Hat{\theta_{D}}^{(1,1)} \\
    \Hat{\theta_{D}}^{(2,1)} \\
    \vdots \\
    \Hat{\theta_{D}}^{(C,1)}
\end{array}\right],T_{D}(\Hat{\boldsymbol{L}})\boldsymbol{F}\left[\begin{array}{c}
    \Hat{\theta_{D}}^{(1,2)} \\
    \Hat{\theta_{D}}^{(2,2)} \\
    \vdots \\
    \Hat{\theta_{D}}^{(C,2)}
\end{array}\right],...,T_{D}(\Hat{\boldsymbol{L}})\boldsymbol{F}\left[\begin{array}{c}
    \Hat{\theta_{D}}^{(1,C)} \\
    \Hat{\theta_{D}}^{(2,C)} \\
    \vdots \\
    \Hat{\theta_{D}}^{(C,C)}
\end{array}\right]\right],\notag\\
=&\sum_{d=0}^{D}\left[T_{d}(\Hat{\boldsymbol{L}})\boldsymbol{F}\left[\begin{array}{c}
    \Hat{\theta_{d}}^{(1,1)} \\
    \Hat{\theta_{d}}^{(2,1)} \\
    \vdots \\
    \Hat{\theta_{d}}^{(C,1)}
\end{array}\right],T_{d}(\Hat{\boldsymbol{L}})\boldsymbol{F}\left[\begin{array}{c}
    \Hat{\theta_{d}}^{(1,2)} \\
    \Hat{\theta_{d}}^{(2,2)} \\
    \vdots \\
    \Hat{\theta_{d}}^{(C,2)}
\end{array}\right],...,T_{d}(\Hat{\boldsymbol{L}})\boldsymbol{F}\left[\begin{array}{c}
    \Hat{\theta_{d}}^{(1,C)} \\
    \Hat{\theta_{d}}^{(2,C)} \\
    \vdots \\
    \Hat{\theta_{d}}^{(C,C)}
\end{array}\right]\right],\notag\\
=&\sum_{d=0}^{D}T_{d}(\Hat{\boldsymbol{L}})\boldsymbol{F}\cdot\left[\begin{array}{cccc}
    \Hat{\theta_{d}}^{(1,1)} & \Hat{\theta_{d}}^{(1,2)} & \cdots & \Hat{\theta_{d}}^{(1,C)} \\
    \Hat{\theta_{d}}^{(2,1)} & \Hat{\theta_{d}}^{(2,2)} & \cdots & \Hat{\theta_{d}}^{(2,C)} \\
    \vdots & \vdots & \vdots & \vdots \\
    \Hat{\theta_{d}}^{(C,1)} & \Hat{\theta_{d}}^{(C,2)} & \cdots & \Hat{\theta_{d}}^{(C,C)}
\end{array}\right],\notag\\ 
=&\sum_{d=0}^{D}T_{d}(\Hat{\boldsymbol{L}})\boldsymbol{F}\cdot\left[\begin{array}{cccc}
    \sum_{b=0}^{D}\theta_{b}^{(1,1)}T_{d}(x_{b}) & \sum_{b=0}^{D}\theta_{b}^{(1,2)}T_{d}(x_{b}) & \cdots & \sum_{b=0}^{D}\theta_{b}^{(1,C)}T_{d}(x_{b}) \\
    \sum_{b=0}^{D}\theta_{b}^{(2,1)}T_{d}(x_{b}) & \sum_{b=0}^{D}\theta_{b}^{(2,2)}T_{d}(x_{b}) & \cdots & \sum_{b=0}^{D}\theta_{b}^{(2,C)}T_{d}(x_{b}) \\
    \vdots & \vdots & \vdots & \vdots \\
    \sum_{b=0}^{D}\theta_{b}^{(C,1)}T_{d}(x_{b}) & \sum_{b=0}^{D}\theta_{b}^{(C,2)}T_{d}(x_{b}) & \cdots & \sum_{b=0}^{D}\theta_{b}^{(C,C)}T_{d}(x_{b})
\end{array}\right],\notag\\
=&\sum_{d=0}^{D}T_{d}(\Hat{\boldsymbol{L}})\boldsymbol{F}\cdot\left(\sum_{b=0}^{D}T_{d}(x_{b})\cdot\underbrace{\left[\begin{array}{cccc}
    \theta_{b}^{(1,1)} & \theta_{b}^{(1,2)} & \cdots & \theta_{b}^{(1,C)} \\
    \theta_{b}^{(2,1)} & \theta_{b}^{(2,2)} & \cdots & \theta_{b}^{(2,C)} \\
    \vdots & \vdots & \vdots & \vdots \\
    \theta_{b}^{(C,1)} & \theta_{b}^{(C,2)} & \cdots & \theta_{b}^{(C,C)}
\end{array}\right]}_{\text{denote as}\ \boldsymbol{\Theta}_{::b}}\right),\notag\\
=&\sum_{d=0}^{D}T_{d}(\Hat{\boldsymbol{L}})\boldsymbol{F}\left(\sum_{b=0}^{D}T_{d}(x_{b})\boldsymbol{\Theta}_{::b}\right), 
\end{align*}
which is the Eq.~\ref{equation-chebyshev-polynomial-2dgraphconv}.


\section{Experimental Setup}
\label{appendix-experimental-setup}
In this section, we report full experimental details used in Section~\ref{section-empirical-studies} for reproducibility. 

\subsection{Dataset Statistics}
\label{appendix-experimental-setup-dataset-statistics}

We report statistical details of datasets used in Section~\ref{section-empirical-studies}. 
Specifically, details of $10$ medium-sized datasets in Section~\ref{section-empirical-studies-node-classification} are reported in Table~\ref{table-datasets-statistics-medium}. 
On the other hand, details of $8$ challenging datasets (including large-scale and latest heterophilic) are reported in Table~\ref{table-datasets-statistics-challenging}. 
All datasets are achieved with Pytorch Geometric framework~\cite{PyTorchGeometric}. 

\begin{table}[!ht]
  \caption{Statistics of $10$ medium-sized datasets.} 
  \vskip 0.15in
  \label{table-datasets-statistics-medium}
  \centering
  \setlength{\tabcolsep}{7pt}
  \renewcommand\arraystretch{1.0}
  \resizebox{\textwidth}{!}{
  \begin{tabular}{lccccc|ccccc}
  \hline
    \multirow{2}{*}{}
        & \multicolumn{5}{c|}{Homophilic} & \multicolumn{5}{c}{Heterophilic} \\ \cline{2-11}
        & Cora & CiteSeer & PubMed & Computers & Photo & Chameleon & Squirrel & Texas & Cornell & Actor \\ \hline
    \# Nodes & $2708$ & $3327$ & $19,717$ & $13,752$ & $7650$ & $2277$ & $5201$ & $183$ & $183$ & $7600$ \\
    \# Edges &  $5278$ & $4552$ & $44,324$ & $245,861$ & $119,081$ & $31,371$ & $198,353$ & $279$ & $277$ & $26,659$ \\
    \# Features & $1433$ & $3703$ & $500$ & $767$ & $745$ & $2325$ & $2089$ & $1703$ & $1703$ & $932$ \\
    \# Classes & $7$ & $6$ & $5$ & $10$ & $8$ & $5$ & $5$ & $5$ & $5$ & $5$ \\
    \hline
  \end{tabular}}
\end{table}
\begin{table}[!ht]
  \caption{Statistics of $8$ challenging datasets.} 
  \vskip 0.15in
  \label{table-datasets-statistics-challenging}
  \centering
  \setlength{\tabcolsep}{6pt}
  \renewcommand\arraystretch{1.0}
  \resizebox{\textwidth}{!}{
  \begin{tabular}{lcc|ccc|ccc}
    \hline
    \multirow{2}{*}{}
        & \multicolumn{2}{c|}{Large homophilic} & \multicolumn{3}{c|}{Large heterophilic} & \multicolumn{3}{c}{Latest heterophilic} \\ \cline{2-9}
        & Ogbn-arxiv & Ogbn-products &  Penn94  &  Genius  &  Gamers  &  Roman-empire &  Tolokers  &  Amazon-ratings \\ \hline
    \# Nodes & $169,343$ & $2,449,029$ & $41,554$ & $421,961$ & $168,114$ & $22,662$ & $11,758$ & $24,492$  \\
    \# Edges &  $1,166,243$ & $61,859,140$ & $1,362,229$ & $984,979$ & $6,797,557$ & $32,927$ & $519,000$ & $93,050$  \\
    \# Features & $128$ & $100$ & $4814$ & $12$ & $7$ & $300$ & $10$ & $300$  \\
    \# Classes & $40$ & $47$ & $2$ & $2$ & $2$ & $18$ & $2$ & $5$  \\
    \hline
  \end{tabular}}
\end{table}

\subsection{Baseline Implementations}
\label{appendix-experimental-setup-baseline-implementations}
For the implementation of GCN and APPNP, we follow common implementation applied in previous works~\cite{GPRGNN,BernNet-GNN-narrowbandresults-1,ChebNetII,JacobiConv,OptBasisGNN}. 
For the remaining baselines, we resort to the officially released code, accessible via the provided URLs as follows. 
Notably, for GNN-HF/LF, we only report the higher results of two variants on each dataset; for the work~\cite{glognn++}, we only take the most effective model, i.e., GloGNN++, as a baseline. 

\begin{itemize}[leftmargin=*,parsep=2pt,itemsep=2pt,topsep=2pt]
\item GPRGNN: \url{https://github.com/jianhao2016/GPRGNN}
\item GNN-HF/LF: \url{https://github.com/zhumeiqiBUPT/GNN-LF-HF}
\item BernNet: \url{https://github.com/ivam-he/BernNet}
\item ChebNetII: \url{https://github.com/ivam-he/ChebNetII}
\item OptBasis: \url{https://github.com/yuziGuo/FarOptBasis}
\item DSGC: \url{https://github.com/liqimai/DSGC}
\item Spec-GN: \url{https://github.com/qslim/gnn-spectrum}
\item ADC: \url{https://github.com/abcbdf/ADC}
\item JacobiConv: \url{https://github.com/GraphPKU/JacobiConv}
\item GCNII: \url{https://github.com/chennnM/GCNII}
\item PDE-GCN: \url{https://openreview.net/forum?id=wWtk6GxJB2x}
\item Nodeformer: \url{https://github.com/qitianwu/NodeFormer}
\item GloGNN++: \url{https://github.com/RecklessRonan/GloGNN}
\end{itemize}

\subsection{Experimental Settings in Section~\ref{section-empirical-studies-node-classification}}
\label{appendix-experimental-setup-experimental-settings-node-classification}
\par {\bf Hyperparameter settings.} We set wide search ranges for general hyperparameters in all models. 
Specifically, we optimize weight decay on $\left\{0,1e-8,1e-7,...,1e-3,1e-2\right\}$, learning rate on $\left\{1e-4,5e-4,1e-3,...,0.1,...,0.5\right\}$, dropout rate on $\left\{0,0.1,0.2,...,0.9\right\}$ for all models. 
For specific hyperparameters in baselines, we follow the same search ranges used in the original papers. 
For specific hyperparameters in ChebNet2D, we set the layer of MLP $h_{\eta}(\cdot)$ to $2$ with $64$ hidden units, and optimize truncated polynomial order $D$ on $\left\{2,4,8,10,12,16\right\}$. 
To keep fair and comprehensive experiments, we apply grid-search for all models to achieve optimal selections on hyperparameters. 
\par {\bf Training details.} We use \textit{Cross Entropy} as the loss function and Adam optimizer~\cite{Adamoptimizer} to train the models, where an early stopping $200$ and a maximum of $2000$ epochs are set for all datasets.

\subsection{Experimental Settings in Section~\ref{section-empirical-studies-node-classification-large-challenging}}
\label{appendix-experimental-setup-experimental-settings-node-classification-large-challenging}
\par {\bf Hyperparameter settings.} Both general hyperparameters and specific hyperparameters in baselines are same as the pipeline in Section~\ref{appendix-experimental-setup-experimental-settings-node-classification}. 
For specific hyperparameters in ChebNet2D, we optimize the layer number of MLP on $\left\{2,3\right\}$, hidden size of MLP on $\left\{128,256,512,1024,2048\right\}$, truncated polynomial order $D$ on $\left\{2,4,8,10,12,16\right\}$. 
All parameters are optimized with grid-search. 
\par {\bf Training details.} We apply batch training for the experiments on large-scale datasets, where the batch training size (number of nodes) will be optimized on $\left\{10000,50000\right\}$. 
Especially, for the training of ChebNet2D on large-scale datasets, we follow the scaled-up strategy used in~\cite{ChebNetII,OptBasisGNN} and precompute polynomial bases to speed up the training. 
For the training on the rest datasets, the settings of ChebNet2D are the same as Section~\ref{appendix-experimental-setup-experimental-settings-node-classification}. 
The same Cross Entropy loss, Adam optimizer, and training epoch settings are used as Section~\ref{appendix-experimental-setup-experimental-settings-node-classification}. 

\subsection{Experimental Settings in Section~\ref{section-empirical-studies-ablation-study-comparison-on-convolution-paradigms}}
\label{appendix-experimental-setup-experimental-settings-ablation-study-comparison-on-convolution-paradigms}
\par {\bf Hyperparameter settings.} Both general hyperparameters and specific hyperparameters in baselines are the same as the pipeline in Section~\ref{appendix-experimental-setup-experimental-settings-node-classification}. 
Please note that all baselines are re-implemented into same decoupling model architecture with MLP as feature transformation. 
Thus, to keep fairness, we follow the same search range as ChebNet2D in Section~\ref{appendix-experimental-setup-experimental-settings-node-classification} and~\ref{appendix-experimental-setup-experimental-settings-node-classification-large-challenging} for the MLP in baselines, and involve similar batch training for experiments on large datasets, where all parameters are optimized by grid-search. 
\par {\bf Training details.} To keep fair comparison, for experiments on large-scale datasets, we adopt the same scaled-up strategy as~\cite{ChebNetII,OptBasisGNN} to all baselines, and precompute polynomial bases to speed up the training. 
For experiments on the rest datasets, the same settings as ChebNet2D in Section~\ref{appendix-experimental-setup-experimental-settings-node-classification} are used for all models. 
The same Cross Entropy loss, Adam optimizer, and training epoch settings are used as Section~\ref{appendix-experimental-setup-experimental-settings-node-classification}.

\end{document}